\documentclass[11pt,reqno]{amsart}

\usepackage{amsmath,amssymb,amsthm}
\usepackage{amsaddr}
\usepackage{tabularx,graphicx,moreverb}
\usepackage{amsmath,amssymb,amsthm}
\usepackage{here}
\usepackage{bm}
\usepackage{natbib}
\usepackage{booktabs}
\usepackage{enumitem}
\setlist[itemize,enumerate]{itemsep=0pt,parsep=0pt,topsep=0pt,partopsep=0pt}
\usepackage{wrapfig}

\usepackage{fullpage}
\usepackage{mathtools}
\mathtoolsset{showonlyrefs=true}

\usepackage{color,comment}

\newtheorem{lemma}{Lemma}
\newtheorem{theorem}{Theorem}
\newtheorem{assumption}{Assumption}
\newtheorem{proposition}{Proposition}
\newtheorem{definition}{Definition}
\newtheorem{corollary}{Corollary}
\newtheorem{example}{Example}
\newtheorem{remark}{Remark}

\renewcommand{\hat}{\widehat}
\renewcommand{\tilde}{\widetilde}

\newcommand{\R}{\mathbb{R}}

\newcommand{\mA}{\mathcal{A}}

\newcommand{\mP}{\mathcal{P}}

\newcommand{\Tr}{\mathrm{T}}

\newcommand{\mZ}{\mathcal{Z}}

\allowdisplaybreaks

\newcommand{\dd}{ \mathrm{d}}
\newcommand{\Sph}{\mathbb{S}^{d-1}}
\newcommand{\TT}{\mathbb T}

\DeclareMathOperator{\diver}{div}
\DeclareMathOperator{\esssup}{ess sup}

\newcommand{\ip}[2]{\bigl\langle #1,#2\bigr\rangle}

\title{{\Large Anti Mode-Collapse in Mean-Field Transformer \\via Auxiliary Variables}}

\date{}

\allowdisplaybreaks
\let\origmaketitle\maketitle
\def\maketitle{
  \begingroup
  \def\uppercasenonmath##1{} 
  \let\MakeUppercase\relax 
  \origmaketitle
  \endgroup
}

\author{Masaaki Imaizumi$^{1,2,3}$, Masanori Koyama$^1$, Noboru Isobe$^2$, Kohei Hayashi$^1$
}
\address{
$^1$The University of Tokyo, Tokyo, Japan \\
$^2$RIKEN Advanced Intelligene Project, Tokyo, Japan\\
$^3$Kyoto University, Kyoto, Japan
}

\email{imaizumi@g.ecc.u-tokyo.ac.jp}

\begin{document}

\maketitle

\begin{abstract}
We use a mean-field-based transformer model to theoretically investigate how auxiliary variables, such as positional encoding, prevent mode collapse of self-attention mechanisms. The use of mean-field transformers to analyze the properties of self-attention mechanisms has garnered significant attention in recent years due to their ability to comprehensively analyze token interactions. However, analysis of this simple model suggests that mode collapse, where token distributions degenerate to a single point, occurs during long inferences (i.e., many layers), indicating a discrepancy with reality. This study investigates this mean-field transformer model and demonstrates that the introduction of auxiliary variables, such as positional encoding, acts as a counterforce against theoretical mode collapse. Specifically, we show that in the theoretical scheme, the energy-maximizing distribution does not degenerate to a single point; instead, it is characterized by a pushforward of the auxiliary variable distribution, thereby avoiding concentration in the Dirac measure. Our main examples are the positional encoding and the fixed prompt insertion treated as a parallel auxiliary-variable mechanism. Furthermore, we demonstrate that positional encoding and prompt insertion possess universality of representation in the limit, meaning that the limit distribution of inference can exactly represent a wide class of distributions. We also analyze several key properties of positional encoding and metastability, and validate our theoretical results through mathematical experiments.
\end{abstract}

\section{Introduction}

Transformers are neural networks built around self-attention mechanisms  \cite{vaswani2017attention}. 
They have played a central role not only in natural language processing but also across machine learning more broadly. 
Also, recent work on looped or depth-recurrent Transformers, which performs inference by repeatedly iterating through layers that share the same weights, has attracted attention due to its high performance \cite{fan2024looped,xu2024looped,altabaa2025recursive,chen2026depthrecurrent,geiping2025scaling}.
Motivated by their empirical success, a growing theoretical literature seeks to clarify the mathematical principles behind the evolution of token representations \cite{yun2020transformers,cordonnier2020relationship,bao2024localize,cui2024phase,teo2024kernelpca}.
In particular, 
a shared-block regime is investigated from aspects of length generalization, expressive power, recursive latent reasoning, and compositional generalization \cite{fan2024looped,altabaa2025recursive,chen2026depthrecurrent}.

A fruitful viewpoint treats repeated self-attention updates as interacting-particle dynamics and studies their mean-field limits \cite{dutta2021redesigning,geshkovski2023emergence,geshkovski2025mathematical,rigollet2025mean}. In particular, \cite{geshkovski2025mathematical,geshkovski2023emergence} formulate simplified self-attention models as interacting particle systems on the sphere and analyze long-time clustering and limiting configurations. Moreover, \cite{geshkovski2024dynamic} shows that these dynamics can exhibit metastability: although they ultimately converge to a single cluster, they may remain near multi-cluster states for very long time scales. This line of work provides a mathematically precise route from iterative self-attention updates to gradient-driven dynamics governed by an energy functional.

A recurring conclusion of this framework is \emph{degeneracy} or \textit{mode-collapse} of token distributions. In the 
variational analysis above in regard to the dynamics of distributions of tokens, long-time limits and maximizing configurations often collapse to Dirac measures or single clusters \cite{geshkovski2025mathematical,geshkovski2023emergence,geshkovski2024dynamic}. 
In other words, according to these analyses, long inferences performed by (loop-based) Transformers ultimately lead to a collapse to a single point.
While there are differing views on this issue based on experience, oversmoothing is not inevitable \cite{dovonon2024oversmoothing}, and masking together with layer normalization can prevent simple rank-one collapse \cite{wu2024maskslayernorm}, it has yet to be settled.
Based on this situation, we pose the following research questions: \textit{What are effective methods for preventing mode collapse in mean-field-based transformer models? Can we provide a theoretical basis for these methods?}

\subsection{Contributions}

\begin{figure}[t]
    \centering
    \includegraphics[width=0.95\linewidth]{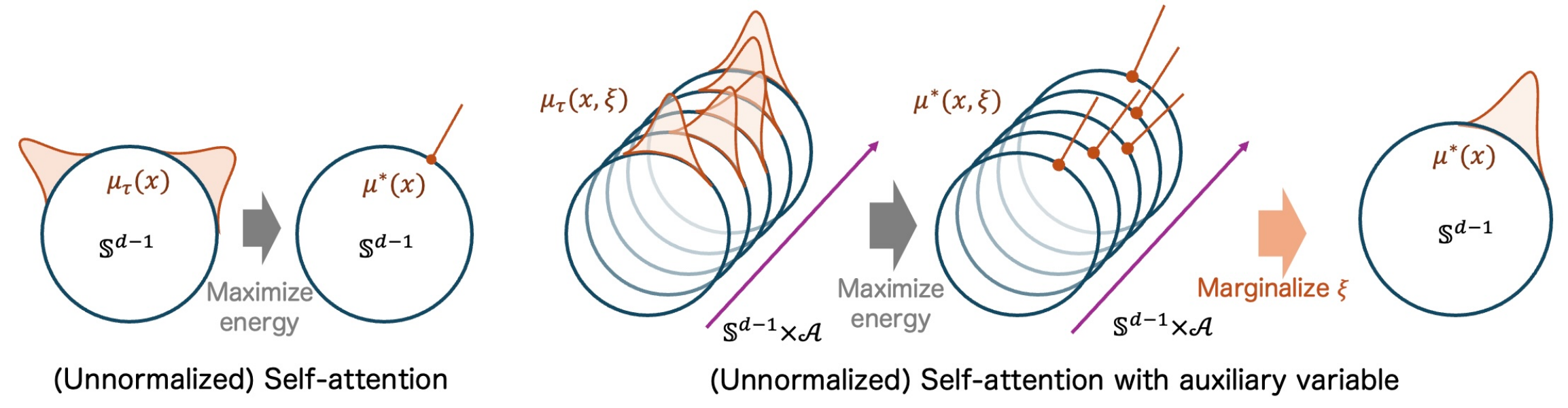}
    \caption{Distributions $\mu_\tau$ of tokens $x$ on $\Sph$ via self-attention and maximizing energy. In the standard USA model (left), an energy maximizer $\mu^*$ is a Dirac measure. In the USA model with auxiliary variables $\xi$ in $\mA$ (right), the addition of an auxiliary variable space prevents $\mu^*$ from degenerating. The energy is maximized conditionally for each value of $\xi$ internally, and the output distribution can be regarded as its marginalization.}
    \label{fig:vis_abst}
\end{figure}

This paper endeavors to fill the aforementioned gap by extending the mean-field theoretical framework such as  \cite{geshkovski2023emergence} to include the variables that are used in actual practice but are often omitted from the analysis. 
In particular, in addition to the dynamics of  \textit{token} variables that has been considered in depth,  we include \textbf{Positional encoding} and \textbf{prefix tokens} inserted at inference time \cite{li2021prefix,lester2021prompt} as tangible probabilistic \textit{auxiliary} variables, and consider the joint law of the token and auxiliary variables with a fixed auxiliary marginal.
More generally, we develop the framework of what we call 
\textbf{Un}normalized \textbf{S}elf-\textbf{A}ttention model with \textbf{A}uxiliary \textbf{V}ariables (\textbf{USA-AV}) and show that, with sufficient regularity, the auxiliary variables prevent the token distribution from collapsing.
Figure~\ref{fig:vis_abst} outlines the result shown by our theory.

The main contributions of this paper are as follows.
\begin{itemize}
    \item \textbf{A fixed-auxiliary variational framework for anti-collapse.}
    We formulate self-attention with auxiliary variables (e.g., positional encodings and prompt prefixes)  in the \textbf{USA-AV} model, which describes the dynamics of the token variable under a fixed auxiliary marginal through joint laws.
    \item \textbf{Conditional Dirac structure with marginalized noncollapse.}
    Under a fixed auxiliary marginal, we show that maximizers may still be Dirac in the conditional content variable at each auxiliary value, while the marginalized content distribution remains non-Dirac. This identifies the precise sense in which auxiliary variables prevent collapse: not by destroying conditional Dirac structure, but by allowing the Dirac location to move with the auxiliary variable.
    \item \textbf{Positional encoding yields orbit-wise exact realization of a target law.}
    For positional variables, we show that RoPE-type kernels already generate spread maximizers, and that arbitrary measurable phase fields go further: once a rotation orbit is fixed, every probability law on that orbit can arise exactly as a maximizing marginalized output law.
    \item \textbf{Prefix tokens upgrade this to gauge-level exact realization of a target law.}
    For prefix tokens, we prove that any $\eta$-pushforward law generated by a measurable map into the sphere can be realized exactly as a maximizing marginalized output law through an orthogonal gauge; when the prompt space is standard Borel and the prompt law is nonatomic, this yields exact realization of arbitrary target laws on $\Sph$.
\end{itemize}

The results above suggest that positional encodings and prefix tokens function not merely as side information, but as genuine \emph{anti mode-collapse mechanisms} inside the idealized shared-weight self-attention model analyzed here.
Most importantly, these anti mode-collapse mechanisms can be designed with appropriate kernels to produce a target distribution through the marginalization of the auxiliary variables. 
Throughout, the phrase \emph{looped transformer-type self-attention dynamics} is meant to emphasize this scope restriction: the theorems concern a shared-weight surrogate, not an arbitrary full Transformer architecture. In Sections~\ref{sec:pe-nondeg} and~\ref{sec:prompt} we refer informally to the last two principles as \emph{orbit-wise} and \emph{gauge-level universality}, respectively, but the actual mathematical statements are exact realization results for maximizing marginals rather than universal approximation statements for sequence-to-sequence maps.

\subsection{Notation}
We use $\|\cdot\|$ for norms, $\langle\cdot,\cdot\rangle$ for Euclidean inner products, $\delta$ for a Dirac measure, $C^1$ for continuously differentiable functions, $\mathcal P(\Omega)$ for probability measures on $\Omega$, and $W_1$ for the $1$-Wasserstein distance. 
For a measurable map \(G:\Omega\to\Omega'\) and a measure \(\nu\in\mathcal P(\Omega)\), we denote by \(G_{\#}\nu\) the pushforward of \(\nu\) by \(G\), defined by \((G_{\#}\nu)(B)=\nu(G^{-1}(B))\) for every Borel set \(B\subset\Omega'\).
Let $\Sph:=\{x\in\mathbb R^d:\|x\|=1\}$ be the unit sphere. We write $\diver_x$ for the surface divergence in the spherical variable, $\sigma_{\Sph}$ for the uniform probability measure on $\Sph$, and $\TT:=\mathbb R/\mathbb Z$ for the one-dimensional torus. We write $\pi_x,\pi_{\mA},\pi_{\mZ}$ for the coordinate projections. $\beta > 0$ denotes the inverse temperature.

\section{Fixed-Auxiliary Variational Framework}\label{sec:framework}

\subsection{Preparation: Mean-Field Framework for Analyzing Transformers}

We first introduce the framework of mean-field transformers, which arises from a standard self-attention block \cite{vaswani2017attention}. For a single-head attention layer with $n$ token inputs, one may write an update of the $i$-th token $x_i^\ell$, $i=1,\dots,n$, at an $\ell$-th layer as $$x_i^{\ell+1}=\mathrm{LN}\left(\sum_{j=1}^n \alpha_{ij}^{\ell} W_Vx_j^{\ell}\right),$$ where $\alpha_{ij}^{\ell}\propto\exp(\langle Qx_i^{\ell},Kx_j^{\ell}\rangle)$, $Q,K,W_V$ are weight matrices, and $\mathrm{LN}(\cdot)$ denotes the (RMS) layer normalization without a scale parameter. 

The interacting-particle and mean-field viewpoints developed for Transformer self-attention \cite{geshkovski2025mathematical,burger2025layernorm} justify the following idealization: take tied layers and a small residual step, use layer normalization to keep $x_i^\ell$ on the sphere, impose a symmetric query--key structure and absorb the value map into coordinates, i.e., $Q=K=W_V=I$, and replace the softmax denominator by a fixed normalization or time rescaling. 
In the resulting continuous-depth limit, the token states follow an interacting system of the form named the \textit{unnormalized self-attention} (USA) model for the $i$-th token $x_i$: $$\dot x_i=P_{x_i}^{\perp} \left(\frac{1}{n}\sum_{j=1}^n \exp(x_j^\top x_i) x_j\right), \qquad x_i \in \Sph.$$
Here, $P_{x_i}^{\perp}$ denotes the projection onto a tangent space of the sphere.
This is exactly the unnormalized self-attention flow used below. 

\subsection{USA-AV Model and its Dynamics} 
Let $(\mA,d_\mA)$ be a compact metric space, and let $\rho$ be a Borel probability measure on $\mA$. 
We set  $X:=\Sph\times \mA$, where the content variable is $x\in\Sph$ and the auxiliary variable is $\xi\in\mA$.
Throughout, we consider admissible joint laws preserving the auxiliary marginal $\rho$.  
In Section~\ref{sec:pe-nondeg}, dedicated to the role of positional encoding, we specialize to $I:=(0,1]$ and set $\mA=I$. 
For Section~\ref{sec:prompt}, dedicated to the role of prefixes, we continue with the idea of fixed marginals and consider a standard Borel prompt space $\mZ$ with Borel probability law $\eta$. 

The point of this section is to record only the minimal dynamical and variational structure needed for the main results. In particular, the interaction kernel is reused along depth/time, so the model should be read as a shared-weight or looped-block surrogate \cite{fan2024looped,xu2024looped,altabaa2025recursive,chen2026depthrecurrent} rather than as a full untied-depth transformer.
When a metric on $X$ is needed, we use $d_X\bigl((x,\xi),(y,\zeta)\bigr):=\|x-y\|+d_\mA(\xi,\zeta)$.

We first specify the class of kernels used throughout the fixed-auxiliary framework.
\begin{definition}[Regular kernel]\label{def:h}
Let $X:=\Sph \times \mA$. A function $h:X\times X\to\mathbb R$ is called a \emph{regular kernel} if it satisfies:
\begin{enumerate}[label=(H\arabic*),ref=H\arabic*]
  \setlength{\parskip}{0cm} 
  \setlength{\itemsep}{0cm} 
  \item\label{H1} {Boundedness and symmetry:}
    $\sup|h|<\infty$, and $h(z,z')=h(z',z)$.
  \item\label{H2} {Continuity:}
    the map $((x,\xi),(y,\zeta))\mapsto h((x,\xi),(y,\zeta))$ is continuous on $X\times X$.
  \item\label{H3} {Regularity of $\nabla_x h$:}
    $\nabla_x h$ is continuous in $(x,y)$ and satisfies that there exist $L_x,M_h<\infty$ such that $\bigl\|\nabla_x h((x,\xi),(y,\zeta))-\nabla_x h((x',\xi),(y,\zeta))\bigr\|\le L_x \|x-x'\|$ and $\|\nabla_x h((x,\xi),(y,\zeta))\|\le M_h$ hold for all $(x,x',y,\xi,\zeta)$.
\end{enumerate}
\end{definition}

Condition~\ref{H3} is the main dynamical assumption: it guarantees that, after averaging over the other particles, the force field remains uniformly Lipschitz in the content variable. Throughout the finite-particle system the auxiliary labels are fixed and only the content variables evolve.

\begin{definition}[USA-AV model: finite-particle ODE]\label{def:particle-ode}
Fix auxiliary labels $\xi_1,\dots,\xi_n\in\mA$. For $i=1,\dots,n$,
\begin{equation}\label{eq:particle-ode}
  \dot x_i(\tau)
  = P_{x_i(\tau)}^{\perp}
      \Bigl[
        \frac1n\sum_{j=1}^{n}
          \nabla_x h \bigl((x_i(\tau),\xi_i),(x_j(\tau),\xi_j)\bigr)
      \Bigr],
  \qquad x_i(0)\in \Sph.
\end{equation}
We write the initial condition as $x_i(0)=x_i^0\in \Sph$ for $i=1,\dots,n$.
\end{definition}

The auxiliary-blind USA model is recovered by choosing a kernel that ignores the auxiliary labels altogether.
\begin{example}[Original USA \cite{geshkovski2025mathematical}]\label{ex:usa}
The standard unnormalized self-attention (USA) model corresponds to the content kernel $$h\bigl((x,\xi),(y,\zeta)\bigr)=\beta^{-1} e^{\beta\langle x,y\rangle}$$. Then the auxiliary labels are passive and \eqref{eq:particle-ode} reduces to the usual USA flow. In the positional specialization this is precisely the familiar position-free baseline.
\end{example}

Two positional specializations, obtained by taking $I:=(0,1]$ and $(\mA,\rho)=(I,\dd s)$, will be used repeatedly later.
\begin{example}[Distance-bias type] \label{ex:relative_distance_pe}
Relative-position and distance-bias mechanisms were introduced into transformer self-attention by \cite{shaw2018self} and later simplified in linear-bias form by ALiBi \cite{press2021alibi}. Given a bounded continuous symmetric positional factor $b_p(s,t)$, we consider $$h\bigl((x,s),(y,t)\bigr)=b_p(s,t)\beta^{-1}\exp \bigl(\beta\langle x,y\rangle\bigr).$$
\end{example}

\begin{example}[RoPE type] \label{ex:rope}
Rotary positional embeddings were introduced in RoFormer \cite{su2021roformer}. Under rotational positional encoding (RoPE), the kernel takes the form $$h\bigl((x,s),(y,t)\bigr)=\beta^{-1}\exp \bigl(\beta\langle x, R_{\theta(t)-\theta(s)} y\rangle\bigr),$$ with $\theta(s)=\omega s$. Here $R_\theta$ is the orthogonal transformation that rotates by angle $\theta$ on a fixed two-dimensional subspace $E\subset\mathbb R^d$ and acts as the identity on $E^\perp$.
\end{example}

Basic ODE well-posedness and the empirical-law weak formulation are standard under Definition~\ref{def:h}; for completeness they are recorded as Propositions~\ref{prop:invariance} and~\ref{prop:weak-n} in Appendix~\ref{app:framework-prelim}.

The empirical law of the particle system is $\mu^{(n)}_{\tau}:= \frac1n\sum_{i=1}^n \delta_{(x_i(\tau),\xi_i)}\in\mathcal P(X)$. When the empirical auxiliary law $\rho_n:=\frac1n\sum_{i=1}^n \delta_{\xi_i}$ converges to $\rho$, the natural fixed-marginal mean-field class is $\mathcal P_{\rho}:=\{\mu\in\mathcal P(X): \mu(\dd x,\dd \xi)=\mu^\xi(\dd x)\rho(\dd \xi)\}$. Since $X$ is a standard Borel space, every $\mu\in\mathcal P_{\rho}$ admits the disintegration $\mu(\dd x,\dd \xi)=\mu^\xi(\dd x)\rho(\dd \xi)$. This fixed-marginal decomposition is the basic structural constraint throughout the paper.

\begin{definition}[Mean-field equation for USA-AV]\label{def:mf-eq}
A time-dependent measure $\mu_\tau\in\mathcal P(X)$ is called a mean-field solution if it satisfies
nonlocal continuity equation
\begin{align}    \label{eq:mf-weak}
\partial_\tau \mu_\tau
+
\diver_x\bigl(\mu_\tau V[\mu_\tau]\bigr)=0
\quad\text{on }X,
\end{align}
where $V[\mu_\tau](x,\xi):= P_x^{\perp}\bigl[\int_{X}\nabla_x h\bigl((x,\xi),(y,\zeta)\bigr)   d \mu_\tau(y,\zeta)\bigr]$ and 
where no divergence is taken in the auxiliary variable.  
\end{definition}

Because the vector field acts only on the content coordinate, the auxiliary label is frozen, while the conditional content law at each label is transported on the sphere by the mean attention force generated by all other conditional laws.
In particular, any solution starting in $\mathcal P_\rho$ remains in $\mathcal P_\rho$ for all times.

Under strengthened Lipschitz hypotheses on $\nabla_x h$ in both arguments, stated in Appendix~\ref{app:framework-prelim}, the mean-field equation is globally well posed and satisfies a standard Dobrushin-type stability estimate. In particular, empirical laws converge to the mean-field solution when the initial data converge in $W_1$; see Proposition~\ref{prop:mean-field-limit}.

\subsection{Energies under a fixed auxiliary marginal}

To extend the variational arguments as in \cite{geshkovski2023emergence} to our framework with $\mP_\rho$, we need only to establish the corresponding mean-field energy and Lyapunov monotonicity.
We provide the explicit form of the corresponding energy and monotonicity claims for the discretely supported $\mu_\tau^{(n)}$ in the Appendix~\ref{app:framework-energy}.
In this section we provide two general results to be used later.

\begin{definition}[Mean-field kernel energy]\label{def:EMF}
For $\mu\in\mathcal P(X)$, we define the energy as $$\mathcal E_h[\mu]:= \frac12 \iint_{X\times X} h(z,z')   d \mu(z)   d \mu(z').$$
\end{definition}

When $\mu(\dd x,\dd \xi)=\mu^\xi(\dd x)\rho(\dd \xi)$, this functional simply averages pairwise interactions over the auxiliary labels and the conditional content laws. The point is that the mean-field equation in Definition~\ref{def:mf-eq} is an ascent flow for this functional.

\begin{proposition}[Monotonicity of the mean-field energy]\label{prop:mf-energy}
Along solutions of \eqref{eq:mf-weak}, $\frac{\mathrm d}{\mathrm d\tau}\mathcal E_h[\mu_\tau]=\int_X\|P_x^{\perp}F_\tau(x,\xi)\|^2 d\mu_\tau(x,\xi)\ge0$, where $F_\tau(x,\xi):=\int_X\nabla_x h((x,\xi),(y,\zeta)) d\mu_\tau(y,\zeta)$.
\end{proposition}

\section{Anti-Collapse via Positional Encoding}\label{sec:pe-nondeg}

We now specialize the general USA-AV framework to positional variables. Thus throughout this section $I:=(0,1]$, $X:=\Sph\times I$, and $\mathcal P_I:=\{\mu\in\mathcal P(X): \mu(\dd s,\dd x)=\dd s \mu^s(\dd x)\}$. For periodic positional kernels we identify $I$ with the torus $\TT$ by gluing the endpoints. In this section, we discuss two points: (i) RoPE kernels admit exact maximizing laws whose conditional slices remain Dirac while the marginalized content law is genuinely non-Dirac, and (ii) once a rotation orbit is fixed, the marginalized maximizing law represents measures along that orbit. 

\subsection{RoPE constructions and orbit-wise exact realization}\label{subsec:pe-main}

We consider a special case of the RoPE-type embedding in Example~\ref{ex:rope}: fix $\beta>0$ and $\omega\in\mathbb R$, and consider $h_R\bigl((x,s),(y,t)\bigr):=\frac1\beta\exp \bigl(\beta \langle x, R_{\omega(t-s)} y\rangle\bigr)$; write $\mathcal E_R:=\mathcal E_{h_R}$.
For a measurable path $x:I\to\Sph$, we write $\Gamma[x]:=(s\mapsto(s,x(s)))_{\#}ds$ for its graph law.

\begin{theorem}[RoPE anti-collapse]\label{thm:rope}
For any $u\in \Sph$, we define
\begin{align}
    x^\ast(s):=R_{-\omega s}u, \mbox{~~and~~}\mu^\ast:=\Gamma[x^\ast].
\end{align}
Then, $\mathcal E_R[\mu]\le\mathcal E_R[\mu^\ast]=e^\beta/(2\beta)$ holds for all $\mu\in \mathcal P_I$,
hence $\mu^\ast$ is a global maximizer on $\mathcal P_I$. Its marginalized content law is written as $\bar\mu^\ast=(x^\ast)_{\#}ds$. 
\end{theorem}

Thus RoPE already gives an exact maximizing configuration whose output law is spread along an orbit rather than collapsing to a single point. The choice of $u$ is arbitrary, and changing the component of $u$ orthogonal to the active rotation plane changes the orbit without changing the maximal value. The next theorem shows that the orbit law need not be uniform: every probability law on the chosen orbit can be realized exactly.

\begin{corollary}[Maximizer with RoPE]\label{cor:latitude}
In the setting of Theorem~\ref{thm:rope}, assume moreover that $d=3$ and $\omega=2\pi m$ for some $m\in\mathbb Z\setminus\{0\}$.
Choose a vector $u \in \Sph$ whose polar angle is $\theta\in[0,\pi]$. Then $x^\ast(s)=R_{-\omega s}u$ is still a maximizer, and $\bar\mu^\ast$ becomes the uniform measure on the small circle at latitude $\theta$.
\end{corollary}

More generally, fix a two-dimensional subspace $E\subset\mathbb R^d$, and for each $\theta\in\mathbb R$ let $R_\theta\in O(d)$ denote the orthogonal transformation that rotates by angle $\theta$ on $E$ and acts as the identity on $E^\perp$. Given a measurable antisymmetric phase kernel $g:I\times I\to\mathbb R$, set $h_g((x,s),(y,t)):= \frac{1}{\beta}\exp (\beta \langle x, R_{g(s,t)} y\rangle)$ with $g(t,s)=-g(s,t)$, and write $\mathcal E_g:=\mathcal E_{h_g}$. Our goal is to realize prescribed orbit-type maximizers $x^\psi(s):=R_{\psi(s)}u$ for a fixed $u\in\Sph$ and arbitrary measurable phase fields $\psi:I\to\mathbb R$.

\begin{theorem}[Orbit-wise exact realization and uniqueness]\label{thm:realize-psi}
Fix $u\in\Sph$ and let $\mathcal O_u:=\{R_\theta u:\theta\in\mathbb R\}$ be its rotation orbit. Then the following hold.
\begin{enumerate}
\item[(i)] \emph{Construction.} Given a measurable function $\psi:I\to\mathbb R$, set $g(s,t):= -\bigl(\psi(t)-\psi(s)\bigr)\ (\mathrm{mod}\ 2\pi)$ and $x^\psi(s):=R_{\psi(s)}u$, so that the antisymmetry condition holds. Then, 
\begin{equation}\label{eq:mu-star}
  \mu^\ast := \Gamma[x^\psi]
\end{equation}
satisfies $\mathcal E_g[\mu] \le \mathcal E_g[\mu^\ast] = ({2\beta})^{-1} e^\beta$ for all $\mu\in\mathcal P_I$. Hence $\mu^\ast$ is a global maximizer and saturates the upper bound.
\item[(ii)] \emph{Exact realization.} For every $\nu\in\mathcal P(\mathcal O_u)$, there exists a measurable $\psi:I\to\mathbb R$ such that the corresponding maximizer \eqref{eq:mu-star} satisfies $\bar\mu^\ast := (\pi_x)_{\#}\mu^\ast = (x^\psi)_{\#}ds = \nu$.
\item[(iii)] \emph{Uniqueness.} If, for $\psi : I \to \R$, a maximizer $\mu\in\mathcal P_I$ satisfies $\mathcal E_g[\mu]=({2\beta})^{-1}e^\beta$, then there exists a measurable map $s\mapsto x(s)\in\Sph$ such that $\mu=\Gamma[x]$ and $x(t)=R_{\psi(t)-\psi(s)}x(s)$ a.e. in $(s,t)$. Hence, there exists $u'\in\Sph$ such that $x(s)=R_{\psi(s)}u'$ holds for a.e. $s$.
\end{enumerate}
\end{theorem}

Part~(ii) precisely describes the phrase \emph{orbit-wise universality}: once the orbit $\mathcal O_u$ is fixed, every probability law on that orbit can occur exactly as a maximizing marginalized output law. Positional encoding therefore does not merely exhibit one noncollapsed maximizer.

\subsection{Energy maximization with conditional Dirac measures}\label{subsec:pe-proof-outline}

We present the principle that explains the results obtained above, namely the result of energy maximization using the conditional Dirac distribution.
In the auxiliary-variable setting of Section~\ref{sec:framework}, once the auxiliary marginal $\rho$ is fixed and nonatomic, the diagonal in $\mA\times\mA$ is negligible and the energy is affine in each conditional slice $\mu^\xi$. 
Maximizers can therefore be Diracized auxiliary value by auxiliary value, and the problem reduces to maximizing over graph laws. 

\begin{theorem}[Conditional Dirac measure]\label{thm:dirac-max}
Under Conditions~\ref{H1}--\ref{H2}, assume in addition that $\rho$ is nonatomic. Then $\mathcal E_h[\cdot]$ attains its maximum on $\mathcal P_{\rho}$. Moreover, there exists a measurable map $x^\ast:\mA\to \Sph$ such that $$\mu^\ast(\dd x,\dd \xi)=\delta_{x^\ast(\xi)}(\dd x)\rho(\dd \xi)$$ is a maximizer. In addition, for the potential $\Phi^{\ast}_\xi(x):=\int_{\mA}\int_{\Sph}h\bigl((x,\xi),(y,\zeta)\bigr)d\mu^{\ast \zeta}(y)\rho(\dd \zeta)$, one has $x^\ast(\xi) \in \arg\max_{x\in \Sph} \Phi^{\ast}_\xi(x)$ for $\rho$-a.e.\ $\xi$.
\end{theorem}

This theorem does \emph{not} say that the marginalized content law must be a Dirac mass. It only says that, at each fixed auxiliary value $\xi$, one may choose an optimizing conditional law that is Dirac. If the maximizing point $x^\ast(\xi)$ varies with $\xi$, then the marginalized law $\bar\mu(\dd x):=\int_{\mA} \delta_{x^\ast(\xi)}(\dd x) \rho(\dd \xi)$ can be continuous or otherwise nondegenerate.

Specializing Theorem~\ref{thm:dirac-max} to $(\mA,\rho)=(I,\dd s)$, the optimization reduces to the pathwise problem
\begin{equation}\label{eq:J-path}
  \max_{x:I \to \Sph}\tilde{\mathcal{E}}_h[x(\cdot)],
  \qquad
  \tilde{\mathcal{E}}_h[x(\cdot)]
  :=\frac12\int_0^1\int_0^1
     h\bigl((x(s),s),(x(t),t)\bigr) ds dt.
\end{equation}

This gives a short proof outline for Theorems~\ref{thm:rope} and~\ref{thm:realize-psi}. In both cases, $h\bigl((x,s),(y,t)\bigr)\le e^\beta/\beta$ pointwise, so $\mathcal E_h[\mu]\le e^\beta/(2\beta)$ for every admissible law. The candidate paths are then chosen so that this upper bound is attained for every pair $(s,t)$: for RoPE, $R_{\omega(t-s)}x^\ast(t)=x^\ast(s)$; for the phase-field construction, $R_{g(s,t)}x^\psi(t)=x^\psi(s)$. Hence every interaction term is exactly $e^\beta/\beta$, so the reduced energy in \eqref{eq:J-path} attains the maximal value $e^\beta/(2\beta)$, and therefore the original mean-field energy is maximized as well. 

The same conditional-Dirac reduction is also the starting point for the separated-kernel analysis below. If $h\bigl((x,s),(y,t)\bigr)=b(s,t) k(x,y)$ with $k$ positive definite and $\Phi(x):=k(x,\cdot)$ the feature map into an RKHS $\mathcal H$, then by Theorem~\ref{thm:dirac-max} any maximizer in the positional setting may be written as $\mu^\ast=\Gamma[x^\ast]$, and with $m(s):=\Phi(x^\ast(s))$ one has $\mathcal E_h(\mu^\ast)=\frac12\int_I\int_I b(s,t)\langle m(s),m(t)\rangle_{\mathcal H} ds dt$ and $\|m(s)\|\equiv R$.
Thus the maximization reduces to arranging a unit-vector field against the spectrum and sign pattern of $b$, which is precisely the viewpoint used in the concrete examples of Section~\ref{sec:examples} and the supplementary separated-kernel statements in Appendix~\ref{app:proofs-spectral}.

\section{Fixed Prompts as Auxiliary Variables}\label{sec:prompt}

This section records prefix tokens as a second concrete realization of the same fixed-auxiliary mechanism. The auxiliary variable is now an externally chosen prompt $z\in\mZ$ with fixed law $\eta$, while only the content variable is optimized. The main point is that the prompt law itself can be transported directly into a maximizing marginalized output law through an orthogonal gauge.

\subsection{Main fixed-prompt realization theorem}\label{sec:fixed-prompt}

Given a fixed prompt distribution $\eta\in \mP(\mZ)$, set $X_{\mZ}:=\Sph\times\mZ$ and define the admissible class $\mP_{\eta}:=\{\mu\in \mP(X_{\mZ}): (\pi_{\mZ})_{\#}\mu=\eta\}$. Thus any $\mu\in\mP_{\eta}$ admits a disintegration $\mu(\dd x,\dd z)=\mu^z(\dd x)\eta(\dd z)$, and the prompt marginal is frozen exactly as the positional marginal was in Section~\ref{sec:framework}. For a measurable map $G:\mZ\to\Sph$, write $\Gamma_\eta[G]:=(z\mapsto(G(z),z))_{\#}\eta$. In this prompt setting we use the distinguished prompt energy notation on $X_{\mZ}$, namely $\mathcal E_h^{\rm pr}[\mu]:=\frac12\iint_{\mZ\times \mZ}\iint_{\Sph\times \Sph}h\bigl((x,z),(y,w)\bigr)\mu^z(\dd x)\mu^w(\dd y)\eta(\dd z)\eta(\dd w)$.

Given a measurable map $\Psi:\mZ\to O(d)$ and $\beta>0$, consider the kernel
\begin{equation}\label{eq:def-h-prompt}
  h_{\Psi}\bigl((x,z),(y,w)\bigr)
  :=
  \frac1\beta
  \exp\Bigl(
    \beta \ip{x}{\Psi(z)\Psi(w)^\Tr y}
  \Bigr).
\end{equation}
We also write $\mathcal E_\Psi:=\mathcal E_{h_\Psi}^{\rm pr}$.

The next theorem is the main result of the section. It is the prompt analogue of the phase-field construction from Section~\ref{sec:pe-nondeg}, but also a genuine strengthening: the maximizing law is no longer confined to a single one-parameter orbit.

\begin{theorem}[Prefix tokens and exact realization]\label{thm:prompt-gauge}
Fix $\eta\in \mP(\mZ)$, $\beta>0$, and $u\in \Sph$. Then the following hold.
\begin{enumerate}
\item[(i)] \emph{Construction.} Given a measurable map $\Psi:\mZ\to O(d)$, let $x^\Psi(z):=\Psi(z)u$ and $\mu^{\ast}:=\Gamma_\eta[x^\Psi]$. Then, for any $\mu\in \mP_{\eta}$, $\mathcal E_{\Psi}[\mu]\le \mathcal E_{\Psi}[\mu^{\ast}]=e^{\beta}/(2\beta)$ holds. Therefore $\mu^{\ast}$ is a global maximizer on $\mP_{\eta}$. Furthermore, its $x$-marginal distribution is given by
\begin{equation}\label{eq:prompt-pushforward}
  \bar\mu^{\ast}
  :=
  (\pi_x)_{\#}\mu^{\ast}
  =
  (x^\Psi)_{\#}\eta
  =
  \bigl(\Psi(\cdot)u\bigr)_{\#}\eta.
\end{equation}
\item[(ii)] \emph{Exact realization.} For every measurable target map $G:\mZ\to\Sph$, there exists a measurable map $\Psi_G:\mZ\to O(d)$ such that $\Psi_G(z)u=G(z)$ for all $z$. Consequently, $\mu^G:=\Gamma_\eta[G]$ is a global maximizer for the kernel $h_\Psi$ with $\Psi=\Psi_G$, and its marginalized content law is exactly $G_{\#}\eta$.
\item[(iii)] \emph{Full exact realization.} If $\mZ$ is a standard Borel space and $\eta$ is nonatomic, then for every $\nu\in\mP(\Sph)$ there exist a measurable target map $G:\mZ\to\Sph$ and a measurable gauge $\Psi_G:\mZ\to O(d)$ such that the maximizer in~(ii) satisfies $(\pi_x)_{\#}\mu^G=G_{\#}\eta=\nu$.
\end{enumerate}
\end{theorem}

Theorem~\ref{thm:prompt-gauge} can be viewed as replacing
$s\mapsto R_{-\omega s}$ in the RoPE case by the more general prompt-dependent orthogonal transformation
$z\mapsto \Psi(z)$. Therefore the fixed reference measure $\eta(\dd z)$ plays the same role as the reference measure $\rho(\dd \xi)$ in USA-AV, specialized to $\dd s$ in the positional sections. The difference is that positional encoding realizes laws on a one-dimensional orbit, whereas prefix tokens allow those on general spaces and hence an arbitrary $\eta$-pushforward law on the sphere.

At the variational level, prefix tokens again reduce to a fixed-marginal maximization problem. The prompt-side Diracization principle is recorded in Appendix~\ref{app:prompt-aux}, and the finite-prompt discrete analogue is recorded in Appendix~\ref{app:proofs-gauge-prompt}.

\section{Examples}\label{sec:examples}

We give several explicit positional kernels as examples and the maximizing marginalized laws they produce. Throughout this section, we keep the positional setting $I=(0,1]$ with uniform reference law $\dd s$, and for concreteness we write the orbit formulas in $d=3$. 
The first and third examples are consequences of the separated-kernel analysis recorded in Appendix~\ref{app:proofs-spectral}, while the second is the canonical RoPE construction from Theorem~\ref{thm:rope}.

\begin{example}[Symmetric distance-bias kernel]\label{ex:nonnegative-separated}
Fix $\lambda_{\rm pos}>0$ and consider the explicit distance-bias kernel $h_{\lambda_{\rm pos}}\bigl((x,s),(y,t)\bigr)= e^{-\lambda_{\rm pos}|s-t|} \beta^{-1} e^{\beta\langle x,y\rangle}$. This is a symmetric distance-bias positional weighting: nearby positions interact more strongly, but every interaction remains attractive.  Proposition~\ref{prop:nonnegative-positional-kernel} in Appendix implies that every maximizing path is constant a.e. Hence every maximizer has the form $\mu^\ast=\Gamma[s\mapsto u]$ for some $u\in\Sph$, and the maximizing marginalized content law is exactly $\bar\mu^\ast=\delta_u$. Thus even a nontrivial concrete positional bias fails to prevent collapse as long as the positional factor stays pointwise nonnegative.
\end{example}

\begin{example}[Single-frequency RoPE]\label{ex:rope-output-law}
Take $d=3$, let the active rotation plane be $E=\mathrm{span}\{e_1,e_2\}$, and choose one full turn across the interval, $\omega=2\pi$. For a polar angle $\theta\in[0,\pi]$, set $u^\theta:=(\sin\theta,0,\cos\theta)\in\Sph$. Then the RoPE kernel $h_R\bigl((x,s),(y,t)\bigr)=\beta^{-1}\exp\bigl(\beta\langle x,R_{2\pi(t-s)}y\rangle\bigr)$ is maximized by the path $x^\theta(s)=R_{-2\pi s}u^\theta=\bigl(\sin\theta\cos(2\pi s),-\sin\theta\sin(2\pi s),\cos\theta\bigr)$. Hence the marginalized output law is $\bar\mu^\theta=(x^\theta)_{\#}\dd s$, namely the uniform measure on the latitude circle $C^\theta=\{(\sin\theta\cos\phi,\sin\theta\sin\phi,\cos\theta):\phi\in[0,2\pi)\}$.
When $\theta=\pi/2$, this becomes the uniform law on the equatorial great circle; when $\theta=0$ or $\pi$, the orbit degenerates to a Dirac mass.
\end{example}

\begin{example}[Cosine Toeplitz kernel]\label{ex:toeplitz-output-law}
Consider the explicit sign-changing Toeplitz kernel $h_c\bigl((x,s),(y,t)\bigr)=\cos(2\pi(t-s))\langle x,y\rangle$. Its Fourier coefficients satisfy $\widehat c(\pm1)=\frac12$ and $\widehat c(m)=0$ for $m\neq \pm1$.  Proposition~\ref{prop:toeplitz-linear} in Appendix therefore yields the maximizing path $x^c(s)=\bigl(\cos(2\pi s),\sin(2\pi s),0,\dots,0\bigr)$. Consequently, $\bar\mu^c=(x^c)_{\#}\dd s$ is the uniform measure on the great circle $\Sph\cap\mathrm{span}\{e_1,e_2\}$. 
\end{example}

These three examples summarize the positional message in concrete terms. Auxiliary variables do not automatically prevent collapse; what matters is the geometry of the coupling. A pointwise nonnegative distance-bias kernel still gives a Dirac maximizer, single-frequency RoPE gives a uniform latitude-circle law, and a cosine Toeplitz kernel gives a uniform great-circle law.

\section{Experiments}\label{sec:numerics}

\subsection{Measuring collapse through many layers}\label{subsec:exp-basic-mechanism}

We experimentally validate whether auxiliary variables prevent mode-collapse in the attractive dynamics. In particular, we compare the original USA model as a baseline with two gauge-structured auxiliary models, RoPE and prefix tokens.

We updated finite $n$ particles $\{x_i(\tau)\}_{i=1}^n$ using three different methods: (i) the standard USA model, (ii) the RoPE-type USA-AV model, and (iii) the USA-AV model with a fixed prompt, and measured the degree of distribution collapse using the metric $\mathsf G_x(\tau)
  :=1-\frac{1}{n(n-1)}\sum_{i\neq j}\ip{x_i(\tau)}{x_j(\tau)}$. The closer $\mathsf G_x$ is to zero, the closer the distribution is to collapse to a single point. The updates were performed up to time $\tau_f=20$. Detailed settings are described in Section~\ref{subsec:detail_exp1}.

\begin{figure}[t]
  \centering
  \includegraphics[width=\linewidth]{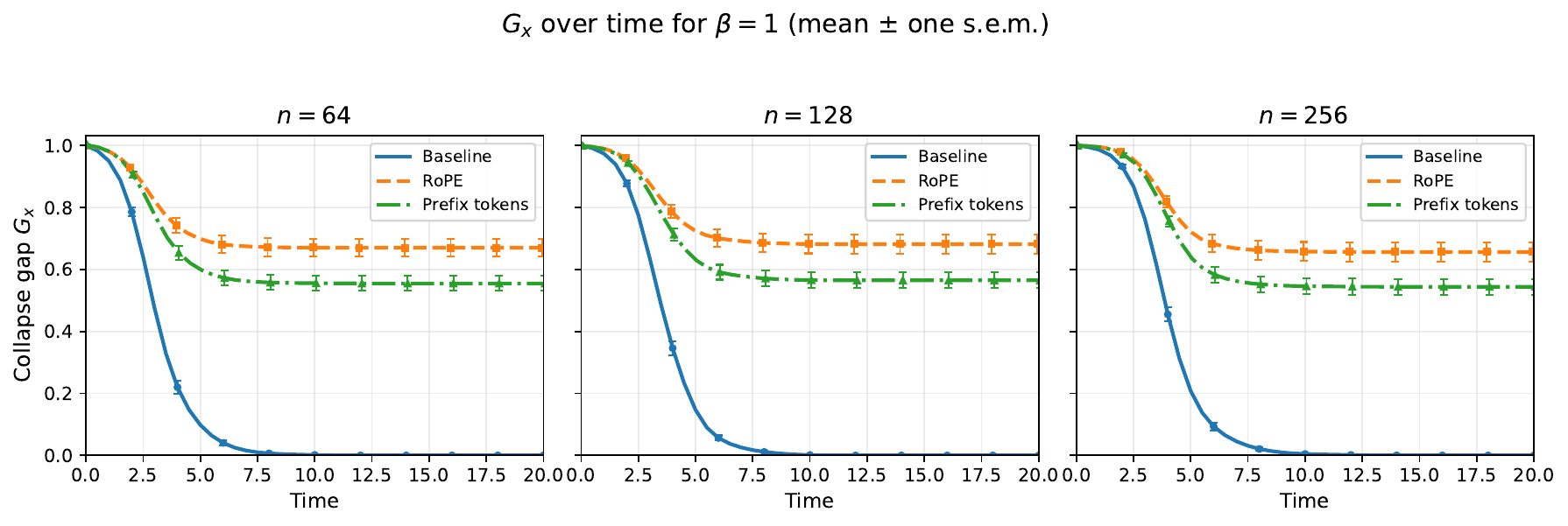}
  \caption{ 
  The mean collapse gap $\mathsf G_x(\tau)$ against the time $\tau$ with $100$ independent random seeds for the random inputs. The lines show the mean and the error bars show the standard error of the mean of the repetition.
  While the baseline USA (no auxiliary variables) collapses to a single content point, RoPE and prefix tokens avoid the mode collapse. }
  \label{fig:exp1-gx-beta1}
\end{figure}

Figure~\ref{fig:exp1-gx-beta1} shows a clear separation between the baseline USA model without auxiliary variables and the USA-AV model.  
At $\tau=20$, the baseline has numerically collapsed in the original frame for all three system sizes.  In contrast, the RoPE and fixed-prompt runs remain noncollapsed and the uncertainty denotes one standard error over seeds.

\subsection{Representation of distributions}
\label{subsec:exp-boundary-metastability}

This experiment examines the representation of the distribution of $n=256$ particles generated by the USA-AV model with various kernels on $\Sph$ with $d=3$. In particular, we study six scenarios: (i) a baseline USA model without auxiliary variables, (ii) a nonnegative distance-bias kernel, (iii) a Toeplitz spectral kernel, (iv) the standard RoPE kernel, (v) a generalized RoPE kernel, and (vi) a USA-AV model with prefix tokens. 
Details are provided in Section~\ref{sec:detail_exp2}.

\begin{figure}[htbp]
    \centering
\includegraphics[width=0.65\linewidth]{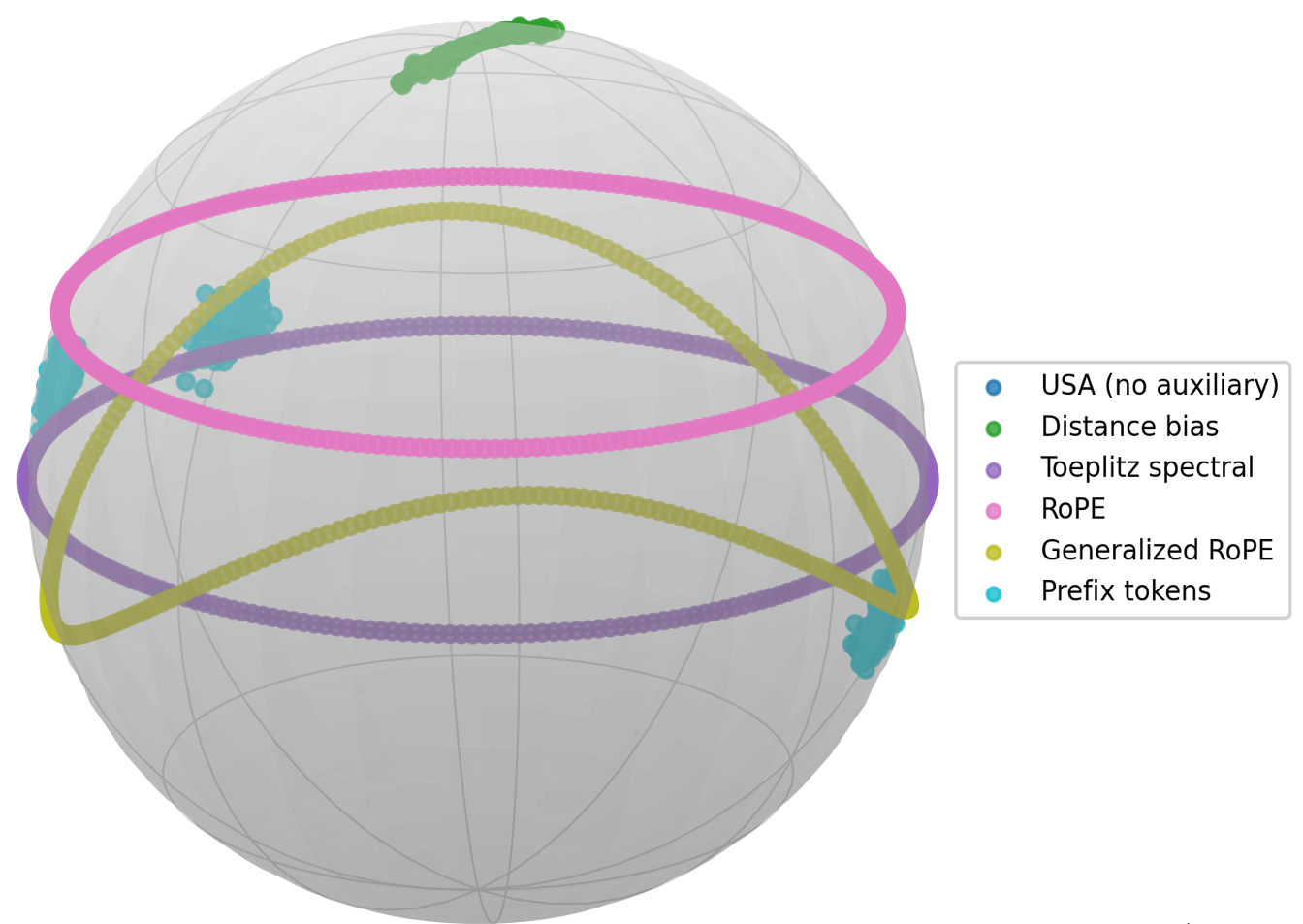}
  \caption{Final particle distributions on $\Sph$. 
  This overlay illustrates the spectral dichotomy and shows how auxiliary-variable structure controls whether mode collapse occurs.\label{fig:exp3_sphere_comparison}
  }
\end{figure}

Figure~\ref{fig:exp3_sphere_comparison} shows the final distributions of the particles on a common sphere for all six cases.  The baseline and distance-bias kernels collapse to a Dirac mass at the north pole, as expected. The Toeplitz spectral kernel produces a great-circle orbit, consistent with the dominance of the first Fourier mode.  The standard RoPE kernel yields a uniform latitude circle.  The generalized RoPE kernel generates a more complicated wavy closed curve, illustrating exact realization of nontrivial orbits.  The fixed-prompt scenario produces a multi-cluster distribution corresponding to metastable clusters.  These outcomes confirm the spectral dichotomy and show how different auxiliary geometries either enforce collapse or yield nontrivial marginal laws.

\section{Conclusion and Future Directions}\label{sec:conclusion}

We developed an auxiliary variational framework for the mean-field
transformer model. By optimizing joint laws for tokens and auxiliary variables, we describe its update through conditional Diracization as a structural anti-collapse
mechanism. In particular, we showed that RoPE and prefix tokens prevent mode collapse and yield exact realization of
maximizing laws.  
These results suggest that collapse in simplified mean-field transformer models partly reflects an auxiliary-free
model class.

A limitation of this study is its limited treatment of dynamics.
While we provide a brief dynamical metastability analysis in Section~\ref{sec:metastability}, we need a more concrete connection between the anti-mode collapse effect and the inference dynamics of transformers.
Future work should connect this variational picture to training dynamics and characterize which noncollapsed output laws are useful for representation and
downstream computation.

\appendix

\section{Related Work}

\subsection{Shared-weight and looped transformer architectures}
A growing body of work studies transformers that reuse the same block across several computation steps. Looped Transformers improve length generalization on algorithmic tasks \cite{fan2024looped}, their expressive power and the role of timestep encoding are analyzed in \cite{xu2024looped}, recursive latent-space reasoning is studied in \cite{altabaa2025recursive}, and depth-recurrent transformers have recently been proposed for compositional generalization \cite{chen2026depthrecurrent}. These papers differ in goals and training setups, but they share the architectural motif most relevant to the present work: repeated application of a shared-weight attention block. Our model should be read as a continuous-depth mean-field surrogate of that regime, not as a theorem about arbitrary untied-depth transformer stacks.

\subsection{Dynamical and mean-field analyses of self-attention}
Another line of work interprets repeated self-attention updates as interacting-particle or mean-field dynamics \cite{dutta2021redesigning,geshkovski2023emergence,geshkovski2025mathematical,isobe2026training}. Within this framework, clustering and dynamic metastability have been studied in \cite{geshkovski2024dynamic}. Our paper belongs primarily to this dynamical stream. The point of departure is that the standard attractive position-free baseline is collapse-prone, and we ask how fixed auxiliary variables alter both the maximizing laws and the long-time picture.

\subsection{Collapse, oversmoothing, and attention sinks}
Recent work studies several degeneration phenomena in transformer models. Oversmoothing is not inevitable \cite{dovonon2024oversmoothing}, masking and layer normalization can qualitatively change the long-depth behavior \cite{wu2024maskslayernorm}, and localization-type collapse has also been analyzed in simplified settings \cite{bao2024localize}. Attention sinks arise in standard softmax language models as disproportionate attention to initial or dedicated sink tokens \cite{xiaoefficient}. That phenomenon is conceptually adjacent but mathematically different from the present paper: our collapse statements concern shared-weight attractive mean-field dynamics and Dirac concentration of marginalized content laws, rather than first-token sink behavior in untied-depth softmax language models.

\subsection{Positional encodings, prefix tokens, and kernel viewpoints}
Practical transformers use a wide variety of positional mechanisms. After the original absolute encodings of \cite{vaswani2017attention}, later work developed relative position representations and segment-level recurrence \cite{shaw2018self,dai2019transformerxl}, relative-bias formulations such as T5 and DeBERTa \cite{raffel2020t5,he2021deberta}, linear-bias schemes such as ALiBi \cite{press2021alibi}, and rotational encodings such as RoPE \cite{su2021roformer}. For long contexts these are often combined with sparse or local attention patterns, such as Longformer, BigBird, and Reformer \cite{beltagy2020longformer,zaheer2020big,kitaev2020reformer}. On the prompting side, in-context conditioning and prefix/prompt tuning show that prepended tokens can steer model behavior \cite{brown2020gpt3,li2021prefix,lester2021prompt,liu2021ptuningv2}. Our contribution is not a new practical PE or prompting recipe, but a variational theory explaining how fixed auxiliary marginals can prevent marginalized collapse. The energy viewpoint is also related to kernel theory: kernel mean embeddings give a language for functionals of probability measures \cite{muandet2017kernelmean}, indefinite kernels can support richer geometries than positive-definite ones \cite{ong2004nonpositive}, and Toeplitz positional kernels naturally invite Fourier analysis \cite{rudin1962fourier,gray2006toeplitz}.

\section{Details of experiments}

\subsection{Common experimental setup}

We solve the finite-particle system \eqref{eq:particle-ode} with fixed auxiliary labels $\xi_i$.  Unless otherwise specified, the numerical solver is a projected second-order Runge--Kutta method with step size $\Delta\tau\in\{10^{-2},5\cdot10^{-3}\}$, followed by the normalization $x_i\leftarrow x_i/\|x_i\|$ after each step.  The terminal time is specified in each individual experiment, or the run is stopped earlier if the relative energy increment over a window of length $1$ is below $10^{-8}$.  Unless otherwise specified in an individual experiment, each reported statistic is averaged over $100$ random seeds, and the main runs are repeated for $n\in\{64,128,256,512\}$ to check finite-particle stability.

\subsection{Details of Section~\ref{subsec:exp-basic-mechanism}} \label{subsec:detail_exp1}

\paragraph{Data-generating model.}
We work on $\Sph$ with $d=3$ and fix the inverse-temperature parameter to $\beta=1$.  We use three system sizes, $n\in\{64,128,256\}$.  In each case there are $m=4$ particles at each auxiliary value and hence $L=n/4$ auxiliary values.  For the baseline and RoPE runs, the auxiliary variable is position,
\[
  s_\ell=\frac{\ell-1}{L},\qquad \ell=1,\dots,L.
\]
For the fixed-prompt run, the auxiliary variable is a prompt label $z_\ell$ with prompt phase $\theta_\ell=2\pi(\ell-1)/L$ and gauge $\Psi(z_\ell)=R_{\theta_\ell}$, where $R_\theta$ rotates the first two coordinates and leaves the third coordinate fixed.  The initial gauge-frame states $q_{\ell,a}(0)$ are sampled independently and uniformly from $\Sph$ and then mapped back to the original frame.  Thus the RoPE initialization is
\[
  x_{\ell,a}(0)=R_{-\omega s_\ell}q_{\ell,a}(0),\qquad \omega=2\pi,
\]
and the fixed-prompt initialization is $x_{\ell,a}(0)=\Psi(z_\ell)q_{\ell,a}(0)$.  For each random seed, the same gauge-frame cloud $\{q_{\ell,a}(0)\}$ is used in all three conditions.

\paragraph{Models compared.}
The three kernels are
\begin{align}
  h_0((x,\xi),(y,\zeta))
  &=\frac1\beta e^{\beta\ip{x}{y}},\label{eq:exp1-baseline-kernel}\\
  h_R((x,s),(y,t))
  &=\frac1\beta\exp\{\beta\ip{x}{R_{\omega(t-s)}y}\},\label{eq:exp1-rope-kernel}\\
  h_\Psi((x,z),(y,w))
  &=\frac1\beta\exp\{\beta\ip{x}{\Psi(z)\Psi(w)^\Tr y}\}.
  \label{eq:exp1-prompt-kernel}
\end{align}
The first kernel ignores the auxiliary labels.  The second is the RoPE kernel of Theorem~\ref{thm:rope}.  The third is the fixed-prompt gauge kernel of Theorem~\ref{thm:prompt-gauge}.  In gauge variables, $q_i=R_{\omega s_i}x_i$ for RoPE and $q_i=\Psi(z_i)^\Tr x_i$ for prefix tokens, the latter two dynamics reduce to the same attractive dynamics as the baseline.  This makes the comparison a direct test of whether collapse occurs in the gauge frame or in the original content frame.

\paragraph{Execution and evaluation.}
We integrate the finite-particle ODE with step size $\Delta\tau=10^{-2}$ up to $\tau_f=20$ and record the trajectory every $0.5$ time units.  The experiment is repeated over $100$ independent random seeds.  We report the collapse gap \begin{equation}\label{eq:exp-Gx-new}
  \mathsf G_x(\tau)
  :=1-\frac{1}{n(n-1)}\sum_{i\neq j}\ip{x_i(\tau)}{x_j(\tau)}.
\end{equation} as the main observable, together with the gauge-frame gap $\mathsf G_q(\tau)$, the conditional diameter $\mathsf D_{\rm cond}(\tau)$, and the relative energy gap $\Delta_E(\tau)$ as diagnostic checks.  For the RoPE and fixed-prompt models, conditional collapse predicts
\[
  \mathsf G_q(\tau_f)\approx0,
  \qquad
  \mathsf D_{\rm cond}(\tau_f)\approx0,
  \qquad
  \Delta_E(\tau_f)\approx0,
\]
while the original-frame gap $\mathsf G_x(\tau_f)$ remains positive because different auxiliary values rotate the common gauge-frame Dirac center to different content locations.

\subsection{Details of Section~\ref{subsec:exp-boundary-metastability}} \label{sec:detail_exp2}

\paragraph{Set-up and models.}
All scenarios use the mean-field USA-AV dynamics described in Section~\ref{sec:framework} with the same projected solver and a common inverse temperature $\beta=1$ for ease of comparison.  We fix the particle count at $n=256$ and sample initial tokens uniformly on $\Sph$.  Positions $s\in\TT$ are drawn uniformly unless stated otherwise.  The individual models are as follows:
\begin{enumerate}[leftmargin=1.2em]
  \item \textbf{Baseline (no auxiliary variables).}  This is the purely attractive USA model from Section~\ref{sec:framework}, which is known to collapse to a Dirac mass.
  \item \textbf{Nonnegative distance-bias.}  We use the separated kernel
    \[
      h_b((x,s),(y,t))
      = b(s,t)\frac{1}{\beta}\exp\bigl(\beta\langle x,y\rangle\bigr),\qquad
      b(s,t)=\epsilon+\exp \bigl(-{\rm dist}_{\TT}(s,t)^2/(2\ell^2)\bigr),
    \]
    with $\epsilon=0.02$ and length scale $\ell=0.1$.  This kernel should collapse to a Dirac law despite the positional labels.
  \item \textbf{Toeplitz spectral.}  We take $h_T((x,s),(y,t))=c(t-s) \langle x,y\rangle$ with $c(\Delta)=\cos(2\pi m\Delta)$ and $m=1$, so that the first Fourier mode dominates.  The theory predicts an $m$-fold circle-type orbit on $\Sph$.
  \item \textbf{Standard RoPE.}  The rotary positional encoding corresponds to the exponentiated kernel $h_{\rm RoPE}((x,s),(y,t))=\beta^{-1}\exp\{\beta\langle R_{-\omega s}x, R_{-\omega t}y\rangle\}$ with $\omega=2\pi$.  In the maximizer, the tokens lie on a latitude circle determined by the effective gauge.
  \item \textbf{Generalized RoPE.}  To illustrate that RoPE can realize more general orbits, we allow the rotary angle to depend on position via a sinusoidal phase field $\alpha(s)$ and use $h_\alpha((x,s),(y,t)):=\beta^{-1}\exp\{\beta\langle x,R_{\alpha(t)-\alpha(s)}y\rangle\}$.  This produces a wavy closed curve on the sphere.
  \item \textbf{Prefix tokens.}  We insert $K_{\rm pr}=3$ fixed prompt tokens with prescribed positions and embeddings, choose gauges $\Psi(z_i)$, and use the prompt kernel \eqref{eq:exp1-prompt-kernel}.  This induces three well-separated clusters and tests the metastability mechanism of Theorem~\ref{thm:main}.
\end{enumerate}

\section{Dynamic complement: Metastability under positional kernels}\label{sec:metastability}

We consider the distance-bias kernel $h\bigl((x,s),(y,t)\bigr)= b(s,t)\frac1\beta e^{\beta\langle x,y\rangle}$ with $b\ge0$, and assume that the initial configuration is organized into $K\ge2$ well-separated clusters $C_1,\dots,C_K$. More precisely, cluster $C_p$ starts inside a spherical cap of radius $r_0$ around a unit vector $\bar u_p$, and the initial inter-center angles are bounded below by $2\sigma_0>0$. Let $\Delta_p(\tau):=\max_{i,j\in C_p}\arccos\langle x_i(\tau),x_j(\tau)\rangle$ be the diameter of cluster $p$, let $u_p(\tau)$ be the normalized cluster average, and let $\Theta_{pq}(\tau):=\arccos\langle u_p(\tau),u_q(\tau)\rangle$ be the angle between cluster centers. We denote by $B_\parallel$ and $B_\times$ the effective intra- and inter-cluster couplings, and by $w_p=|C_p|/n$ and $W_{pq}$ the corresponding coarse-grained weights; the exact definitions are given in Appendix~\ref{app:proofs-metastability}.

\begin{theorem}[Metastability for nonnegative separated kernels]\label{thm:main}
Assume the setting above. For sufficiently large $\beta$, the finite-particle positional specialization of USA-AV exhibits two well-separated time scales:
\begin{enumerate}
\item \emph{Fast intra-cluster contraction.}
There exist constants $c_1,c_2,C_1>0$ such that $\Delta_p(\tau)\le c_1 r_0 e^{-c_2\beta B_\parallel \tau}+C_1 e^{-\beta(1-\cos\sigma_0)}$ for every cluster $p$. In particular, each cluster becomes essentially point-like by $T_f=O \left((\beta B_\parallel)^{-1}\log r_0^{-1}\right)$.
\item \emph{Metastable trapping.}
If $r\in (0,\sigma_0/4)$ and $\beta \ge (1-\cos\sigma_0)^{-1}\log \bigl(8B_\times/(B_\parallel r)\bigr)$, define
\[
T_m:=\inf\{\tau\ge T_f:\max_p\Delta_p(\tau)>2r\ \text{ or }\ \min_{p\ne q}\Theta_{pq}(\tau)<\sigma_0\}.
\]
Then for all $\tau\in[T_f,T_m]$ one has $\Delta_p(\tau)\le 2r$ and $\Theta_{pq}(\tau)\ge \sigma_0$ for $p\ne q$, while $T_m\ge c_3(B_\parallel/B_\times)e^{\beta(1-\cos\sigma_0)}$ for some $c_3>0$.
\item \emph{Slow reduced dynamics.}
On $[T_f,T_m]$, the cluster centers follow the coarse-grained system
\[
\begin{aligned}
  \dot u_p
  &=P_{u_p}^\perp \left[\sum_{q=1}^{K} w_q W_{pq} e^{\beta\langle u_p,u_q\rangle}u_q\right]+R_p(\tau),
  \|R_p(\tau)\|\le C_2\bigl(r+e^{-\beta(1-\cos\sigma_0)}\bigr).
\end{aligned}
\]
Hence the observed $K$-cluster configuration is accurately described by a reduced $K$-point flow up to exponentially small error.
\item \emph{First-merger time.}
The time until the first cluster merger is of order $e^{\beta\delta_{\rm gap}}$, where $\delta_{\rm gap}$ is the initial angular gap of the reduced configuration.
\end{enumerate}
\end{theorem}

This theorem makes precise the picture that the static maximizer alone misses. Even when the energy landscape eventually favors one-direction aggregation, the dynamics first contracts within each cluster and then remains for an exponentially long time near a reduced multi-cluster configuration. The detailed assumptions, constants, and proof are given in Appendix~\ref{app:proofs-metastability}.

\section{Technical material for Section~\ref{sec:framework}}\label{app:tech-framework}

\subsection{Statements moved from Section~\ref{sec:framework}}\label{app:framework-prelim}

\begin{proposition}[Preservation of the sphere constraint and global solvability]\label{prop:invariance}
Under Condition~\ref{H3} of Definition~\ref{def:h}, for any initial datum $(x_1^0,\dots,x_n^0)\in(\Sph)^n$ and any fixed auxiliary labels $\xi_1,\dots,\xi_n\in\mA$, equation~\eqref{eq:particle-ode} admits a unique solution
$(x_1,\dots,x_n)\in C^1 \big([0,\infty);(\Sph)^n\big)$ on $[0,\infty)$. Moreover, $\|x_i(\tau)\|\equiv1$ for all $\tau>0$.
\end{proposition}

\begin{proposition}[Weak form of the time evolution]\label{prop:weak-n}
For any $\varphi\in C^1(X)$,
\[
  \frac{\mathrm d}{\mathrm d\tau}
  \int_{X}\varphi(x,\xi)   d \mu^{(n)}_{\tau}(x,\xi)
  =
  \int_{X}
     \bigl\langle \nabla_x\varphi(x,\xi), V^{(n)}_{\tau}(x,\xi)\bigr\rangle
        d \mu^{(n)}_{\tau}(x,\xi),
\]
where the velocity field is
\[
  V^{(n)}_{\tau}(x,\xi)
  := P_{x}^{\perp}\Bigl[
        \int_{X}
          \nabla_x h\bigl((x,\xi),(y,\zeta)\bigr)
             d \mu^{(n)}_{\tau}(y,\zeta)
     \Bigr] \in T_x\Sph.
\]
\end{proposition}

For the stability estimate used in the mean-field limit, we assume the strengthened Lipschitz bounds on $\nabla_x h$ with respect to its first and second arguments:
\begin{align}
\bigl\|\nabla_x h((x,\xi),(y,\zeta))-\nabla_x h((x',\xi'),(y,\zeta))\bigr\|
&\le L_x d_X\bigl((x,\xi),(x',\xi')\bigr), \tag{3}\label{eq:first-arg-lip}\\
\bigl\|\nabla_x h((x,\xi),(y,\zeta))-\nabla_x h((x,\xi),(y',\zeta'))\bigr\|
&\le L_\mu d_X\bigl((y,\zeta),(y',\zeta')\bigr). \tag{4}\label{eq:second-arg-lip}
\end{align}

\begin{proposition}[Stability inequality and convergence]\label{prop:mean-field-limit}
Let $h$ be a regular kernel satisfying Conditions~\ref{H1}--\ref{H3} and the additional Lipschitz bounds \eqref{eq:first-arg-lip}--\eqref{eq:second-arg-lip}. Assume that at initial time $\tau=0$, one has $\mu_0 \in \mathcal P_{\rho}$ and $W_1(\mu^{(n)}_0,\mu_0)\to0$. Then for every $T>0$ there exists a constant $C_T$ such that
\[
  \sup_{0\le\tau\le T}
  W_1\bigl(\mu^{(n)}_\tau, \mu_\tau\bigr)
   \le e^{C_T}
       W_1\bigl(\mu^{(n)}_0, \mu_0\bigr)
        + o(1), \qquad n\to\infty.
\]
Here $\mu_\tau$ denotes the unique solution of \eqref{eq:mf-weak}.
\end{proposition}

\subsection{Finite-particle and empirical-law statements}

\begin{proof}[Proposition~\ref{prop:invariance}]
\textbf{(1) Local existence and uniqueness.}
Let $x_{1:n}=(x_1,\dots,x_n)\in\R^{nd}$ and define
\[
F_i(x_{1:n}):=P_{x_i}^\perp \Big[\tfrac1n\sum_{j=1}^n
\nabla_x h\big((x_i,\xi_i),(x_j,\xi_j)\big)\Big],\quad F=(F_1,\dots,F_n):\R^{nd} \to \R^{nd}.
\]
By the Lipschitz and boundedness parts of Condition~\ref{H3} and Lemma~\ref{lem:proj},
\[
\|F(x_{1:n})-F(y_{1:n})\|
\le \big(L_x + 2M_h\big) \|x_{1:n}-y_{1:n}\|.
\]
Hence $F$ is locally Lipschitz. Consider the Picard iteration for the integral equation
$x_{1:n}(T)=x_{1:n}(0)+\int_0^T F(x_{1:n}(\tau))   d\tau$, namely
$x_{1:n}^{(k+1)}(T):=x_{1:n}(0)+\int_0^T F\big(x_{1:n}^{(k)}(\tau)\big)d\tau$.
For sufficiently small $T>0$,
\begin{align}
    \|x_{1:n}^{(k+1)}(T)-x_{1:n}^{(k)}(T)\|_\infty & = \left\|\int_0^T F(x_{1:n}^{(k)}(\tau))   d\tau-\int_0^T F(x_{1:n}^{(k-1)}(\tau))   d\tau\right\|_\infty\\
    & \leq L_x T \|x_{1:n}^{(k)}(T)-x_{1:n}^{(k-1)}(T)\|_\infty.
\end{align}
If $T$ is such that $L_x T < 1$, the contraction mapping theorem implies that $x_{1:n}^{(k)}(T)$ converges as $k\to\infty$. This yields local existence and uniqueness up to time $T$ (see Theorem 2.2 in \cite{teschl2012ordinary}).

\textbf{(2) Invariance of the sphere.} For each $i$,
\begin{align}
\frac{\mathrm d}{\mathrm d\tau}\|x_i\|^2
&= 2 \langle x_i(\tau), \dot{x}_i (\tau) \rangle \\
&=  2\left\langle x_i, P_{x_i}^{\perp}\left[\frac1n\sum_{j=1}^{n}
          \nabla_x h \bigl((x_i(\tau),\xi_i),(x_j(\tau),\xi_j)\bigr)\right]\right\rangle \\
          &=0
\end{align}
so $\|x_i(\tau)\|\equiv1$ for all $\tau>0$. The last equality follows from the defining property of $P_{x_i}^{\perp}$. Therefore the solution remains in $(\Sph)^n$.

\textbf{(3) Boundedness and continuation principle.} Since $(\Sph)^n$ is compact and the bound $\|\nabla_x h\|\le M_h$ from Condition~\ref{H3} together with boundedness of the projection implies $\|F(X)\|\le M_h$, the maximal existence time cannot be finite. Indeed, if $T_{\max}<\infty$, then $\{x_{1:n}(\tau)\}_{\tau<T_{\max}}$ remains in $(\Sph)^n$, where $F$ is locally Lipschitz, and the Picard iteration can be restarted from $\tau=T_{\max}$, contradicting maximality. Hence $T_{\max}=\infty$.
\end{proof}

This simple invariance statement plays two roles in the rest of the paper: it justifies the sphere-constrained dynamics as a genuine ODE on a compact manifold, and it lets us estimate all later quantities without tracking radial errors.

The only auxiliary estimate needed in the proof above is the Lipschitz dependence of the tangent projection on the base point.
\begin{lemma}[Lipschitz continuity of the tangent projection]\label{lem:proj}
Let $P_x^\perp:=I-xx^\top$. For $x,x'\in \Sph$,
$\|P_x^\perp-P_{x'}^\perp\|\le 2 \|x-x'\|$.
\end{lemma}

\begin{proof}[Lemma~\ref{lem:proj}]
It follows from
$\|xx^\top-x'x'^\top\|\le \|x\|\|x-x'\|+\|x-x'\|\|x'\|\le 2\|x-x'\|$.
\end{proof}

\begin{proof}[Proposition~\ref{prop:weak-n}]
By definition,
\(
 \frac{\mathrm d}{\mathrm d\tau}\frac1n\sum_i \varphi(x_i,\xi_i)
 = \frac1n\sum_i \langle \nabla_x\varphi(x_i,\xi_i),\dot x_i\rangle.
\)
Substituting \eqref{eq:particle-ode} and rewriting the result in terms of $\mu^{(n)}_{\tau}$ gives the claim.
\end{proof}

\subsection{Auxiliary well-posedness material}

This subsection collects the precise weak-solution formulation behind Definition~\ref{def:mf-eq}. The strategy is standard but worth making explicit: we first define weak solutions, then study frozen linearized flows, and finally close the nonlinear problem by a contraction argument on curves of measures. This discussion is summarized in Proposition~\ref{prop:flow-wp}.

Let the state space be
\[
  M:=X=\Sph\times\mA,
\]
where $\Sph\subset\R^d$ ($d\ge 2$) is the sphere and $(\mA,d_\mA)$ is the compact metric auxiliary space introduced in Section~\ref{sec:framework}. Then $M$ is compact. We equip it with the distance
\[
  d_X\bigl((x,\xi),(y,\zeta)\bigr):=\|x-y\|+d_\mA(\xi,\zeta).
\]

Assume that a kernel $h:M\times M\to\R$ is given and is differentiable with respect to the first argument $x$.
We write its derivative as
\[
  \nabla_x h\bigl((x,\xi),(y,\zeta)\bigr)\in\R^d.
\]
Using the orthogonal projection $P_x^\perp$ onto the tangent space of the sphere and the force
\[
  F(x,\xi;\mu):=\int_M \nabla_x h\bigl((x,\xi),(y,\zeta)\bigr) \mu(\dd y,\dd \zeta),
\]
we define the velocity field
\begin{equation}\label{eq:Vmu}
  V[\mu](x,\xi):=P_x^\perp F(x,\xi;\mu)\in T_x\Sph\subset\R^d.
\end{equation}
Since the auxiliary component is frozen, the characteristic flow acts only on the sphere coordinate.

\subsubsection{Mean-field (measure) equation and weak solutions}
For a curve of probability measures $(\mu_\tau)_{\tau\ge 0}\subset\mathcal P(M)$, we write the USA-AV mean-field equation (continuity equation) as
\begin{equation}\label{eq:CE}
  \partial_\tau\mu_\tau + \diver_{x} \bigl(\mu_\tau V[\mu_\tau]\bigr)=0
  \quad\text{on }M,
  \qquad \mu_{\tau=0}=\mu_0\in\mathcal P(M).
\end{equation}
Here $\diver_x$ denotes the surface divergence with respect to the spherical variable $x$.
Since $V[\mu_\tau](\cdot,\xi)\in T_x\Sph$, this equation is intrinsically well defined.

\begin{definition}[Weak solution (measure solution)]\label{def:weak}
A continuous curve $\mu_\cdot\in C([0,T];\mathcal P(M))$ is called a weak solution of \eqref{eq:CE} if, for every $\varphi\in C^1(M)$, the map $\tau\mapsto\int_M\varphi \dd\mu_\tau$ is absolutely continuous and
\begin{equation}\label{eq:weakform}
  \frac{\dd}{\dd \tau}\int_M \varphi(x,\xi) \mu_\tau(\dd x,\dd \xi)
   = 
  \int_M \bigl\langle \nabla_x\varphi(x,\xi), V[\mu_\tau](x,\xi)\bigr\rangle 
  \mu_\tau(\dd x,\dd \xi)
\end{equation}
holds for a.e.\ $\tau\in[0,T]$.
\end{definition}

\subsubsection{Assumptions (strong regularity: Lipschitz type)}

\begin{assumption}[Lipschitz-type assumptions]\label{ass:Lip}
There exist constants $M_h,L_x,L_\mu<\infty$ such that for all $x,x',y,y'\in\Sph$ and $\xi,\zeta,\zeta'\in\mA$,
\begin{enumerate}
\item \emph{Boundedness}: 
$\|\nabla_x h((x,\xi),(y,\zeta))\|\le M_h$.
\item \emph{$x$-Lipschitz continuity}: 
$\|\nabla_x h((x,\xi),(y,\zeta))-\nabla_x h((x',\xi),(y,\zeta))\|
\le L_x\|x-x'\|$.
\item \emph{Lipschitz continuity in the second argument}: 
$\|\nabla_x h((x,\xi),(y,\zeta))-\nabla_x h((x,\xi),(y',\zeta'))\|
\le L_\mu d_X((y,\zeta),(y',\zeta'))$.
\end{enumerate}
\end{assumption}

\begin{remark}
Assumption~\ref{ass:Lip}(iii) immediately yields Lipschitz continuity with respect to $W_1$:
for any $\mu,\nu\in\mathcal P(M)$ and any coupling $\pi\in\Gamma(\mu,\nu)$,
\[
\Big\|\int \nabla_x h(\cdot,z) \dd\mu(z)-\int \nabla_x h(\cdot,z') \dd\nu(z')\Big\|
\le \int L_\mu d_X(z,z') \dd\pi(z,z').
\]
Taking the infimum over $\pi$ gives
\begin{equation}\label{eq:muLip}
\esssup_{(x,\xi)}\|F(x,\xi;\mu)-F(x,\xi;\nu)\|\le L_\mu W_1(\mu,\nu).
\end{equation}
\end{remark}

\subsubsection{Preparatory lemmas: Lipschitz continuity of tangent projections and velocity fields}

\begin{lemma}[Lipschitz continuity of the tangent projection]\label{lem:projLip}
For $x,x'\in\Sph$, one has $\|P_x^\perp-P_{x'}^\perp\|\le 2\|x-x'\|$.
\end{lemma}
\begin{proof}[Lemma~\ref{lem:projLip}]
Since $P_x^\perp-P_{x'}^\perp=x'x'^\top-xx^\top$,
\[
\|xx^\top-x'x'^\top\|\le \|x\| \|x-x'\|+\|x-x'\| \|x'\|=2\|x-x'\|.
\]
\end{proof}

\begin{lemma}[Uniform Lipschitz continuity of the velocity field]\label{lem:V-Lip}
Under Assumption~\ref{ass:Lip}, the following hold:
\begin{enumerate}
\item \emph{Spatial Lipschitz continuity.}
For every $\mu\in\mathcal P(M)$,
\[
  \|V[\mu](x,\xi)-V[\mu](x',\xi)\|
  \le L_V \|x-x'\|,
  \qquad L_V:=L_x+2M_h.
\]
\item \emph{Measure Lipschitz continuity.}
For every $\mu,\nu\in\mathcal P(M)$,
\[
  \esssup_{(x,\xi)\in M}\|V[\mu](x,\xi)-V[\nu](x,\xi)\|
  \le L_\mu W_1(\mu,\nu),
  \qquad \text{where }L_\mu\text{ is the constant in Assumption~\ref{ass:Lip}(iii)}.
\]
\end{enumerate}
\end{lemma}

\begin{proof}[Lemma~\ref{lem:V-Lip}]
(a)
\[
V[\mu](x,\xi)-V[\mu](x',\xi)
=(P_x^\perp-P_{x'}^\perp)F(x,\xi;\mu)+P_{x'}^\perp(F(x,\xi;\mu)-F(x',\xi;\mu)).
\]
By Assumption~\ref{ass:Lip}(i) and Lemma~\ref{lem:projLip},
$\|(P_x^\perp-P_{x'}^\perp)F(x,\xi;\mu)\|\le 2\|x-x'\| M_h$.
By Assumption~\ref{ass:Lip}(ii),
$\|F(x,\xi;\mu)-F(x',\xi;\mu)\|\le L_x\|x-x'\|$.
Therefore
$\|V[\mu](x,\xi)-V[\mu](x',\xi)\|\le(L_x+2M_h)\|x-x'\|$.

(b)
Using $\|P_x^\perp\|=1$ and \eqref{eq:muLip},
$\|V[\mu]-V[\nu]\|\le \|F(\cdot;\mu)-F(\cdot;\nu)\|\le L_\mu W_1(\mu,\nu)$.
\end{proof}

\subsubsection{Flow maps for frozen curves (linearized problem)}

\begin{definition}[Frozen curves and linearized flows]\label{def:frozen}
Fix a continuous curve $\gamma_\cdot\in C([0,T];\mathcal P(M))$ and consider the nonautonomous ODE
\begin{equation}\label{eq:char-frozen}
  \frac{\dd}{\dd \tau}X(\tau)=V[\gamma_\tau](X(\tau),\xi),
  \qquad X(0)=x,\quad \xi\in\mA \text{ fixed}.
\end{equation}
The associated flow map is denoted by $\Phi_\tau^\gamma(x,\xi):=(X(\tau),\xi)$.
\end{definition}

\begin{lemma}[Global well-posedness and Lipschitz continuity of frozen flows]\label{lem:frozen-flow}
Under Assumption~\ref{ass:Lip}, for any frozen curve $\gamma_\cdot$, equation \eqref{eq:char-frozen} has a unique global solution, and $\Phi_\tau^\gamma$ satisfies
\[
  \|\Phi_\tau^\gamma(x,\xi)-\Phi_\tau^\gamma(x',\xi)\|\le e^{L_V\tau}\|x-x'\|
\]
with respect to $x$ (where $L_V$ is the constant from Lemma~\ref{lem:V-Lip}).
Moreover, $\|X(\tau)\|\equiv 1$, so $\Phi_\tau^\gamma$ maps $M$ to itself.
\end{lemma}

\begin{proof}[Lemma~\ref{lem:frozen-flow}]
By Lemma~\ref{lem:V-Lip}(a), $V[\gamma_\tau](\cdot,\xi)$ is uniformly Lipschitz.
Hence Picard-Lindel\"of gives a unique local solution.
Since $\|V[\gamma_\tau]\|\le \|F\|\le M_h$, no finite-time blow-up can occur, hence the solution extends globally.
Spherical invariance follows from
$\frac{\dd}{\dd \tau}\|X(\tau)\|^2=2\langle X(\tau),P_{X(\tau)}^\perp[\cdots]\rangle=0$.
The Lipschitz estimate follows from Gr\"onwall's inequality.
\end{proof}

\subsubsection{Pushforward representation of solutions to the linear continuity equation}

\begin{lemma}[Pushforwards are weak solutions]\label{lem:pushforward-weak}
Let $\gamma_\cdot$ be a frozen curve and define $\mu_\tau:=(\Phi_\tau^\gamma)_\#\mu_0$.
Then $\mu_\cdot$ is a weak solution of
\[
\partial_\tau\mu_\tau+\diver_x(\mu_\tau V[\gamma_\tau])=0
\]
in the sense of Definition~\ref{def:weak}.
\end{lemma}

\begin{proof}[Lemma~\ref{lem:pushforward-weak}]
Take any $\varphi\in C^1(M)$.
By the definition of pushforward,
\[
\int_M\varphi \dd\mu_\tau=\int_M \varphi(\Phi_\tau^\gamma(z)) \dd\mu_0(z).
\]
Differentiating the right-hand side with respect to $\tau$ and using the chain rule together with \eqref{eq:char-frozen}, we get
\[
\frac{\dd}{\dd \tau}\int_M\varphi(\Phi_\tau^\gamma(z)) \dd\mu_0(z)
=\int_M \bigl\langle \nabla_x\varphi(\Phi_\tau^\gamma(z)), V[\gamma_\tau](\Phi_\tau^\gamma(z))\bigr\rangle \dd\mu_0(z).
\]
Writing $\Phi_\tau^\gamma(z)=:z'$ yields the weak form \eqref{eq:weakform}.
\end{proof}

\subsubsection{Construction of solutions to the nonlinear problem: the contraction mapping (flow-map method)}

\begin{definition}[Nonlinear solution operator]\label{def:Tmap}
For $T>0$, define the operator
\[
\mathcal T:\ C([0,T];\mathcal P(M))\to C([0,T];\mathcal P(M)),
\qquad
(\mathcal T\gamma)_\tau := (\Phi_\tau^\gamma)_\#\mu_0.
\]
\end{definition}

\begin{lemma}[Dobrushin-type transition inequality for frozen flows]\label{lem:dobrushin-frozen}
Under Assumption~\ref{ass:Lip}, for any $\gamma,\tilde\gamma\in C([0,T];\mathcal P(M))$ and any $\tau\in[0,T]$,
\begin{equation}\label{eq:T-contraction-kernel}
  W_1\bigl((\mathcal T\gamma)_\tau,(\mathcal T\tilde\gamma)_\tau\bigr)
  \le
  L_\mu\int_0^\tau e^{L_V(\tau-r)} W_1(\gamma_r,\tilde\gamma_r) \dd r,
\end{equation}
where $L_V,L_\mu$ are the constants from Lemma~\ref{lem:V-Lip}.
\end{lemma}

\begin{proof}[Lemma~\ref{lem:dobrushin-frozen}]
Fix $z=(x,\xi)\in M$ and write
$Z^\gamma(\tau;z):=\Phi_\tau^\gamma(z)$ and $Z^{\tilde\gamma}(\tau;z):=\Phi_\tau^{\tilde\gamma}(z)$.
Estimating the difference of the two trajectories with the same initial datum and using Lemma~\ref{lem:V-Lip}, we obtain
\[
\frac{\dd}{\dd \tau}\|X^\gamma(\tau)-X^{\tilde\gamma}(\tau)\|
\le
\|V[\gamma_\tau](X^\gamma(\tau),\xi)-V[\tilde\gamma_\tau](X^{\tilde\gamma}(\tau),\xi)\|
\le L_V\|X^\gamma(\tau)-X^{\tilde\gamma}(\tau)\| + L_\mu W_1(\gamma_\tau,\tilde\gamma_\tau).
\]
(Here the auxiliary variable $\xi$ is the same for both.)
Variation of constants gives
\[
\|X^\gamma(\tau)-X^{\tilde\gamma}(\tau)\|\le
L_\mu\int_0^\tau e^{L_V(\tau-r)}W_1(\gamma_r,\tilde\gamma_r) \dd r.
\]
Now use the coupling $\pi:=({\rm Id}, {\rm Id})_\#\mu_0$ and define
\[
\Pi_\tau := (\Phi_\tau^\gamma,\Phi_\tau^{\tilde\gamma})_\#\mu_0 \in \Gamma\bigl((\mathcal T\gamma)_\tau,(\mathcal T\tilde\gamma)_\tau\bigr).
\]
Then
\[
W_1\bigl((\mathcal T\gamma)_\tau,(\mathcal T\tilde\gamma)_\tau\bigr)
\le \int_M d_X\bigl(\Phi_\tau^\gamma(z),\Phi_\tau^{\tilde\gamma}(z)\bigr) \mu_0(\dd z)
=\int_M \|X^\gamma(\tau;z)-X^{\tilde\gamma}(\tau;z)\| \mu_0(\dd z),
\]
since the $\mA$-component is identical.
Integrating the pointwise estimate proves \eqref{eq:T-contraction-kernel}.
\end{proof}

\begin{proposition}[Global well-posedness via the flow-map method]\label{prop:flow-wp}
Under Assumption~\ref{ass:Lip}, for every initial measure $\mu_0\in\mathcal P(M)$, equation \eqref{eq:CE} has a unique weak solution
$\mu_\cdot\in C([0,\infty);\mathcal P(M))$ on $[0,\infty)$.
Moreover, for any two solutions $\mu_\cdot,\nu_\cdot$,
\begin{equation}\label{eq:W1-stab}
  W_1(\mu_\tau,\nu_\tau)\le e^{(L_V+L_\mu)\tau} W_1(\mu_0,\nu_0)\qquad(\tau\ge 0)
\end{equation}
holds.
\end{proposition}

\begin{proof}[Proposition~\ref{prop:flow-wp}]
\emph{Step 1: short-time existence and uniqueness (contraction mapping).}
Equip $\mathcal X_T:=C([0,T];\mathcal P(M))$ with the metric
$\|\gamma-\tilde\gamma\|_T:=\sup_{\tau\in[0,T]}W_1(\gamma_\tau,\tilde\gamma_\tau)$.
Since $M$ is compact, $\mathcal P(M)$ is complete under $W_1$, hence so is $\mathcal X_T$.

By Lemma~\ref{lem:dobrushin-frozen},
\[
\|\mathcal T\gamma-\mathcal T\tilde\gamma\|_T
\le
\sup_{\tau\le T} L_\mu\int_0^\tau e^{L_V(\tau-r)} \|\gamma-\tilde\gamma\|_T \dd r
=
\kappa(T) \|\gamma-\tilde\gamma\|_T,
\]
where
\[
\kappa(T):=
L_\mu\sup_{\tau\le T}\int_0^\tau e^{L_V(\tau-r)} \dd r
=
\begin{cases}
L_\mu\dfrac{e^{L_VT}-1}{L_V},& L_V>0,\\
L_\mu T,& L_V=0.
\end{cases}
\]
If $T>0$ is chosen small enough so that $\kappa(T)<1$, Banach's fixed point theorem yields a unique fixed point $\mu_\cdot\in\mathcal X_T$ of $\mathcal T$.
By Definition~\ref{def:Tmap} and Lemma~\ref{lem:pushforward-weak}, this fixed point is a weak solution of \eqref{eq:CE}.

\emph{Step 2: global extension.}
Since $V[\mu]$ is uniformly bounded ($\|V[\mu]\|\le M_h$), the solution remains in $\mathcal P(M)$ on every time interval.
Hence one may iterate the short-time argument starting from the endpoint $\mu_T$ and patch together finitely many intervals to obtain a unique solution on any finite interval $[0,\bar T]$.
Letting $\bar T\uparrow\infty$ gives a global solution.

\emph{Step 3: stability and uniqueness.}
For two solutions $\mu,\nu$, viewing each as a frozen curve and applying Lemma~\ref{lem:dobrushin-frozen}, one gets
\[
W_1(\mu_\tau,\nu_\tau)\le L_\mu\int_0^\tau e^{L_V(\tau-r)} W_1(\mu_r,\nu_r)\dd r + e^{L_V\tau}W_1(\mu_0,\nu_0).
\]
(The last term comes from the standard estimate obtained by transporting an initial coupling; equivalently, one solves the differential inequality $\dot\phi\le L_V\phi+L_\mu W_1$.)
Gr\"onwall's inequality yields \eqref{eq:W1-stab}. In particular, if $\mu_0=\nu_0$, then $\mu_\tau=\nu_\tau$.
\end{proof}

\subsection{Proof of Proposition~\ref{prop:mean-field-limit}}

\begin{proof}[Proposition~\ref{prop:mean-field-limit}]
Existence and uniqueness of the mean-field solution under the strengthened hypotheses follow from Proposition~\ref{prop:flow-wp} established above in this appendix. We therefore focus on the quantitative stability estimate, using the classical Dobrushin--Gronwall coupling method \cite{dobrushin1979vlasov,sznitman1991topics}.

We first record the Lipschitz constants of the velocity field. Let
\[
  F_\mu(x,\xi):=\int_X \nabla_x h\bigl((x,\xi),(y,\zeta)\bigr) d\mu(y,\zeta),
  \qquad
  V[\mu](x,\xi):=P_x^\perp F_\mu(x,\xi).
\]
By the bound $\|\nabla_x h\|\le M_h$ from Condition~\ref{H3}, $\|F_\mu(x,\xi)\|\le M_h$ for all $(x,\xi)$ and all $\mu$. Hence, using Lemma~\ref{lem:projLip} together with \eqref{eq:first-arg-lip}, we obtain for every $\mu\in\mathcal P(X)$ and every $(x,\xi),(x',\xi')\in X$,
\begin{align*}
\|V[\mu](x,\xi)-V[\mu](x',\xi')\|
&\le \|(P_x^\perp-P_{x'}^\perp)F_\mu(x,\xi)\|
   +\|P_{x'}^\perp(F_\mu(x,\xi)-F_\mu(x',\xi'))\| \\
&\le 2M_h\|x-x'\| + L_x d_X\bigl((x,\xi),(x',\xi')\bigr) \\
&\le L_V  d_X\bigl((x,\xi),(x',\xi')\bigr),
\end{align*}
where
\[
  L_V:=2M_h+L_x.
\]
Likewise, by \eqref{eq:second-arg-lip}, for every $\mu,\nu\in\mathcal P(X)$ and every coupling $\pi\in\Gamma(\mu,\nu)$,
\begin{align*}
\|F_\mu(x,\xi)-F_\nu(x,\xi)\|
&\le \iint_{X\times X}
   \bigl\|\nabla_x h((x,\xi),z)-\nabla_x h((x,\xi),z')\bigr\|
    d\pi(z,z') \\
&\le L_\mu\iint_{X\times X} d_X(z,z') d\pi(z,z').
\end{align*}
Taking the infimum over $\pi$ and using $\|P_x^\perp\|=1$, we get
\[
  \sup_{(x,\xi)\in X}\|V[\mu](x,\xi)-V[\nu](x,\xi)\|
  \le L_\mu W_1(\mu,\nu).
\]

Now let $\mu_\tau$ and $\nu_\tau$ be two weak solutions of \eqref{eq:mf-weak} with initial data $\mu_0$ and $\nu_0$, respectively. Let $\Phi^\mu_\tau$ and $\Phi^\nu_\tau$ be the associated characteristic flows. For any coupling $\pi_0\in\Gamma(\mu_0,\nu_0)$, define
\[
  \Pi_\tau := (\Phi^\mu_\tau,\Phi^\nu_\tau)_\#\pi_0
  \in \Gamma(\mu_\tau,\nu_\tau).
\]
Write $\Phi^\mu_\tau(z)=(X^\mu_\tau(z),\Xi(z))$ and $\Phi^\nu_\tau(z')=(X^\nu_\tau(z'),\Xi(z'))$, noting that the auxiliary component is frozen along the flow. Then, for $\pi_0$-a.e.\ $(z,z')\in X\times X$,
\begin{align*}
\frac{d}{d\tau} d_X\bigl(\Phi^\mu_\tau(z),\Phi^\nu_\tau(z')\bigr)
&= \frac{d}{d\tau}\|X^\mu_\tau(z)-X^\nu_\tau(z')\| \\
&\le \|V[\mu_\tau](\Phi^\mu_\tau(z))-V[\nu_\tau](\Phi^\nu_\tau(z'))\| \\
&\le L_V  d_X\bigl(\Phi^\mu_\tau(z),\Phi^\nu_\tau(z')\bigr)
   + L_\mu W_1(\mu_\tau,\nu_\tau).
\end{align*}
By variation of constants,
\[
d_X\bigl(\Phi^\mu_\tau(z),\Phi^\nu_\tau(z')\bigr)
\le e^{L_V\tau}d_X(z,z')
   +L_\mu\int_0^\tau e^{L_V(\tau-r)}W_1(\mu_r,\nu_r) dr.
\]
Integrating this inequality against $\pi_0$ and using that $\Pi_\tau$ is a coupling of $\mu_\tau$ and $\nu_\tau$, we obtain
\[
  W_1(\mu_\tau,\nu_\tau)
  \le e^{L_V\tau}  \iint_{X\times X}   d_X(z,z') d\pi_0(z,z')
     + L_\mu\int_0^\tau e^{L_V(\tau-r)}W_1(\mu_r,\nu_r) dr.
\]
Taking the infimum over $\pi_0\in\Gamma(\mu_0,\nu_0)$ yields
\[
  W_1(\mu_\tau,\nu_\tau)
  \le e^{L_V\tau}W_1(\mu_0,\nu_0)
     + L_\mu\int_0^\tau e^{L_V(\tau-r)}W_1(\mu_r,\nu_r) dr.
\]
Applying Gr\"onwall's lemma in integral form, we conclude that
\[
  W_1(\mu_\tau,\nu_\tau)\le e^{(L_V+L_\mu)\tau}W_1(\mu_0,\nu_0)
  \qquad\text{for all } \tau\in[0,T].
\]

Finally, Proposition~\ref{prop:weak-n} shows that the empirical measure curve $\mu^{(n)}_\tau$ is itself a weak solution of \eqref{eq:mf-weak} with initial datum $\mu^{(n)}_0$. Applying the preceding estimate with $\nu_\tau=\mu_\tau$ gives
\[
  \sup_{0\le\tau\le T}W_1\bigl(\mu^{(n)}_\tau,\mu_\tau\bigr)
  \le e^{(L_V+L_\mu)T}W_1\bigl(\mu^{(n)}_0,\mu_0\bigr).
\]
This is stronger than the stated estimate, since the additional $o(1)$ term may be taken equal to $0$. Because $W_1(\mu^{(n)}_0,\mu_0)\to0$, the convergence follows.
\end{proof}

\section{Proofs for the variational reductions in Sections~\ref{sec:framework}--\ref{sec:prompt}}\label{app:proofs-framework-prompt}

\subsection{Statements moved from Section~\ref{sec:framework}}\label{app:framework-energy}

\begin{definition}[Discrete energy with auxiliary labels]\label{def:En}
For $\mu^{(n)}_{\tau}$, define
\[
  \mathcal E_{h}^{(n)}(\mu^{(n)}_{\tau})
  := \frac{1}{2n^2}\sum_{i,j=1}^{n}
      h\bigl((x_i(\tau),\xi_i),(x_j(\tau),\xi_j)\bigr).
\]
\end{definition}
\begin{proposition}[Monotonicity of the discrete fixed-auxiliary energy]\label{prop:energy-production}
Along solutions of \eqref{eq:particle-ode},
\[
  \frac{\mathrm d}{\mathrm d\tau}\mathcal E_{h}^{(n)}(\mu^{(n)}_{\tau})
   =
  \frac1n\sum_{i=1}^n
    \bigl\|P_{x_i}^{\perp} F_i(\mu^{(n)}_{\tau})\bigr\|^2
   \ge 0,
\]
where
\(
  F_i(\mu^{(n)}_{\tau})
  := \frac1n\sum_{j=1}^n
       \nabla_x h\bigl((x_i(\tau),\xi_i),(x_j(\tau),\xi_j)\bigr)
\).
\end{proposition}

\begin{proposition}[Disintegration formula for the mean-field energy]\label{prop:mf-energy-disintegration}
For $\mu \in \mathcal P_{\rho}$,
            \begin{equation}\label{eq:J}
          \mathcal E_h[\mu] =
          \frac12
          \iint_{\mA \times \mA}
          \iint_{\Sph\times \Sph}
              h\bigl((x,\xi),(y,\zeta)\bigr)
              d\mu^\xi(x) d\mu^\zeta(y) \rho(\dd \xi)\rho(\dd \zeta) .
        \end{equation}
\end{proposition}
\subsection{Energy identities and disintegration formulas}

\begin{proof}[Proposition~\ref{prop:energy-production}]
By symmetry,
\(
  \nabla_{x_i}\mathcal E_{h}^{(n)}=(1/n)F_i.
\)
Hence
\begin{align}    
  \frac{\mathrm d}{\mathrm d\tau}\mathcal E_{h}^{(n)}= \sum_{i=1}^n \langle \nabla_{x_i}\mathcal E_{h}^{(n)},\dot x_i\rangle = \frac{1}{n}\sum_{i=1}^n \langle F_i, P_{x_i}^{\perp}F_i\rangle = \frac{1}{n}\sum_{i=1}^n \|P_{x_i}^{\perp}F_i\|^2.
\end{align}
\end{proof}

\begin{proof}[Proposition~\ref{prop:mf-energy-disintegration}]
The proof is straightforward and is omitted.
\end{proof}

\begin{proof}[Proposition~\ref{prop:mf-energy}]
Using symmetry,
\[
\frac{\mathrm d}{\mathrm d\tau}\iint h   d \mu_\tau d \mu_\tau
=2\iint \langle \nabla_x h((x,\xi),(y,\zeta)),V_\tau(x,\xi)\rangle
     d \mu_\tau(x,\xi)d \mu_\tau(y,\zeta).
\]
Substituting $V_\tau=P_x^{\perp}F_\tau$ and using $\langle F, P_x^{\perp}F\rangle=\|P_x^{\perp}F\|^2$ gives the claim.
\end{proof}

\subsection{Proof of Theorem~\ref{thm:dirac-max}}

\begin{proof}[Theorem~\ref{thm:dirac-max}]
\textbf{(1) Existence of a maximizer.}
Because $\Sph$ and $\mA$ are compact, $X$ is compact. Hence $\mathcal P_{\rho}$ is weak-$\ast$ compact, and $\mathcal E_h[\cdot]$ is weak-$\ast$ continuous by \eqref{eq:J} and the continuity of $h$ on the compact space $X\times X$. Therefore the maximum is attained.

\smallskip
\textbf{(2) Linearity at each auxiliary value.}
Since $\rho$ is nonatomic, the diagonal set in $\mA\times\mA$ has zero $(\rho\otimes\rho)$-measure. For any $\mu\in\mathcal P_{\rho}$, set
\[
  \Phi_\xi(x)
  :=\int_{\mA}  \int_{\Sph}
      h\bigl((x,\xi),(y,\zeta)\bigr) d\mu^\zeta(y) \rho(\dd \zeta).
\]
Then $\Phi_\xi(\cdot)$ is continuous on $\Sph$, and $(\xi,x)\mapsto\Phi_\xi(x)$ is measurable by Condition~\ref{H2} and dominated convergence. Hence
\[
  \mathcal E_h(\mu)
  =\frac12\int_{\mA} \Bigl[\int_{\Sph}\Phi_\xi(x) d\mu^\xi(x)\Bigr] \rho(\dd \xi),
\]
so with all other auxiliary values fixed, $\mathcal E_h[\mu]$ is linear in $\mu^\xi$.

\smallskip
\textbf{(3) Measurable maximizing selection and Dirac improvement.}
The set-valued map
$\mathcal A_\Phi(\xi):=\arg\max_{x\in \Sph}\Phi_\xi(x)$
is nonempty and compact-valued by compactness of the sphere, and its graph is measurable (measurable maximum theorem for Carath\'eodory functions). Since $\mA$ is standard Borel, the Kuratowski--Ryll-Nardzewski selection theorem gives a measurable selection $x^\sharp:\mA\to \Sph$ with $x^\sharp(\xi)\in \mathcal A_\Phi(\xi)$. Defining
$\tilde\mu(\dd x,\dd \xi):=\delta_{x^\sharp(\xi)}(\dd x)\rho(\dd \xi)$, one obtains
\[
  \int_{\Sph} \Phi_\xi(x) d\mu^\xi(x)
   \le \Phi_\xi\bigl(x^\sharp(\xi)\bigr)\qquad(\rho\text{-a.e. }\xi),
\]
and therefore $\mathcal E_h(\tilde\mu)\ge \mathcal E_h(\mu)$ after integrating in $\xi$. Equality holds iff $\mu^\xi$ is supported on $\mathcal A_\Phi(\xi)$ for $\rho$-a.e.\ $\xi$.

\smallskip
\textbf{(4) Diracization of maximizers.}
Applying (3) to a maximizer $\mu^\star$ obtained in (1), we get
$\hat\mu(\dd x,\dd \xi)=\delta_{\hat x(\xi)}(\dd x)\rho(\dd \xi)$ with $\mathcal E_h(\hat\mu)\ge \mathcal E_h(\mu^\star)$. By maximality, $\mathcal E_h(\hat\mu)=\mathcal E_h(\mu^\star)$, so $\hat\mu$ is also a maximizer. Writing this maximizer as $\mu^\ast$ proves the claim. The fixed-point condition is simply the statement that $x^\ast(\xi)$ has been chosen measurably in $\arg\max \Phi_\xi^\ast$.
\end{proof}

\subsection{Euler-Lagrange condition for maximizers}

\begin{proposition}[Necessary condition: vanishing of the projected gradient]\label{prop:EL}
Assume in addition that $h$ is $C^1$ in the first variable $x$.
Then any maximizer
$\mu^\ast(\dd x,\dd \xi)=\delta_{x^\ast(\xi)}(\dd x)\rho(\dd \xi)$
of Theorem~\ref{thm:dirac-max} satisfies, for $\rho$-a.e.\ $\xi$,
\[
  P_{x^\ast(\xi)}^{\perp}
  \Biggl[
    \int_{\mA}  \int_{\Sph}
      \nabla_x h\bigl((x^\ast(\xi),\xi),(y,\zeta)\bigr) d\mu^{\ast \zeta}(y) \rho(\dd \zeta)
  \Biggr]
   = 0,
\]
that is, the projected gradient on the tangent space of the sphere vanishes.
\end{proposition}

\begin{proof}[Proposition~\ref{prop:EL}]
Since $x^\ast(\xi)$ is a maximizer of $\Phi^{ \ast}_\xi$, the first-order optimality condition under the constraint $\|x\|=1$ implies
$P_{x^\ast(\xi)}^\perp \nabla_x\Phi^{ \ast}_\xi(x^\ast(\xi))=0$.
Expanding the definition of $\nabla_x\Phi^{ \ast}_\xi$ gives the stated formula.
\end{proof}

\subsection{Supplementary fixed-prompt statements}\label{app:prompt-aux}

\begin{proposition}[Diracization under a Nonatomic Prompt Marginal]\label{prop:prompt-dirac-max}
Assume that $\eta$ is nonatomic.
Then $\mathcal E_h^{\rm pr}[\cdot]$ attains a maximum on $\mP_{\eta}$, and there exists a measurable map
\[
  x^{\ast}:\mZ\to \Sph
\]
such that
\[
  \mu^{\ast}(\dd x,\dd z)=\delta_{x^{\ast}(z)}(\dd x)\eta(\dd z)
\]
is one of its maximizers.
In particular, its $x$-marginal distribution is given by
\[
  \bar\mu^{\ast}:=(\pi_x)_{\#}\mu^{\ast}=x^{\ast}_{\#}\eta.
\]
\end{proposition}

\begin{proof}[Proposition~\ref{prop:prompt-dirac-max}]
The existence of a maximizer follows from the same compactness and continuity argument as in Theorem~\ref{thm:dirac-max}.
Moreover, if $\eta$ is nonatomic, then the diagonal set
\[
  \{(z,w)\in \mZ^2 ; z=w\}
\]
has zero $(\eta\otimes\eta)$-measure.
Hence, in the proof of Theorem~\ref{thm:dirac-max}, if one replaces the abstract auxiliary pair $(\mA,\rho)$ with $(\mZ,\eta)$, the same measurable-selection argument yields a graph-type maximizer
$\mu^{\ast}(\dd x,\dd z)=\delta_{x^{\ast}(z)}(\dd x)\eta(\dd z)$.
The final representation of the marginal distribution follows immediately from the definition.
\end{proof}
\section{Proofs for the explicit constructions in Sections~\ref{sec:pe-nondeg} and~\ref{sec:prompt}}\label{app:proofs-gauge-prompt}

\subsection{Auxiliary rigidity lemma}

This subsection collects a short rigidity fact that is used in several proofs but would interrupt the main narrative if stated earlier. It is elementary, but it clarifies exactly where the rigidity used later comes from.

\begin{lemma}[Independence and degeneration of functional relations]\label{lem:indep-degen}
Let $(X,Y)$ be independent random variables on a measurable space, and let $T$ be injective and measurable.
If $Y=T(X)$ almost surely, then both $X$ and $Y$ are constants (Dirac random variables).
\end{lemma}

\begin{proof}[Lemma~\ref{lem:indep-degen}]
Since $\sigma(Y)\subset\sigma(X)$ and $X$ and $Y$ are independent, $\sigma(Y)$ is independent of itself.
Hence $\sigma(Y)$ is trivial, so $Y$ is constant.
Then $Y=T(X)$ and injectivity of $T$ imply that $X$ is constant as well.
\end{proof}

\subsection{RoPE and phase-field constructions}

\begin{proof}[Theorem~\ref{thm:rope}]
For any $\mu\in\mathcal P_I$, Cauchy--Schwarz gives
\[
  \big\langle x,R_{\omega(t-s)}y\big\rangle\le 1
\]
for all $(x,s),(y,t)\in X$, hence
\[
  h_R\bigl((x,s),(y,t)\bigr)\le \frac1\beta e^\beta .
\]
Integrating against $\mu(ds,dx)\mu(dt,dy)$ yields
\[
  \mathcal E_R[\mu]\le \frac12\cdot \frac1\beta e^\beta .
\]
On the other hand, for $\mu^\ast=\Gamma[x^\ast]$ with $x^\ast(s)=R_{-\omega s}u$,
\[
  \big\langle x^\ast(s),R_{\omega(t-s)}x^\ast(t)\big\rangle
  =\big\langle R_{-\omega s}u, R_{\omega(t-s)}R_{-\omega t}u\big\rangle
  =\langle R_{-\omega s}u, R_{-\omega s}u\rangle
  =1
\]
for all $s,t$. Thus every interaction term attains the upper bound and
\[
  \mathcal E_R[\mu^\ast]=\frac{e^\beta}{2\beta}.
\]
Hence $\mu^\ast$ is a global maximizer. The description of $\bar\mu^\ast$ follows because $s\mapsto R_{-\omega s}u$ is the corresponding orbit-segment pushforward; it parametrizes the full orbit at constant speed when $\omega/(2\pi)$ is an integer over $I$.
\end{proof}

\begin{proof}[Corollary~\ref{cor:latitude}]
Immediate from Theorem~\ref{thm:rope}.
\end{proof}

\subsection{Proof of Theorem~\ref{thm:realize-psi}}

\begin{proof}[Theorem~\ref{thm:realize-psi}]
We first record the pointwise upper bound and its equality condition.

\begin{lemma}[Pointwise upper bound and equality condition]\label{lem:pointwise}
For any $\theta\in\mathbb R$ and $x,y\in\Sph$,
\[
  \frac{1}{\beta}\exp \big(\beta \langle x,R_\theta y\rangle\big)
   \le \frac{1}{\beta}e^\beta,
\]
and equality holds if and only if $\langle x,R_\theta y\rangle=1$, equivalently $y=R_{-\theta}x$.
\end{lemma}

\begin{proof}[Lemma~\ref{lem:pointwise}]
Cauchy-Schwarz gives $\langle x,R_\theta y\rangle\le\|x\|  \|R_\theta y\|=1$. Since the exponential is monotone, the equality condition is exactly equality of the unit vectors.
\end{proof}

\textbf{(i) Construction from a phase field.}
Applying Lemma~\ref{lem:pointwise} to the integrand of the mean-field energy yields
\[
  \mathcal E_g[\mu]
  \le \frac12\iint\frac{1}{\beta}e^\beta   \dd s\dd t
  = \frac{1}{2\beta}e^\beta .
\]
For $\mu^\ast$, however,
\[
  \langle x^\psi(s), R_{g(s,t)} x^\psi(t)\rangle
  = \big\langle R_{\psi(s)}u, R_{\psi(s)-\psi(t)} R_{\psi(t)}u\big\rangle
  = \langle R_{\psi(s)}u, R_{\psi(s)}u\rangle = 1,
\]
so $\mathcal E_g[\mu^\ast]=\frac{1}{2\beta}e^\beta$.

\textbf{(ii) Orbit-wise exact realization.}
If $\mathcal O_u$ is a singleton, then the claim is trivial. Otherwise $\mathcal O_u$ is a small circle. Since $(I,ds)$ is nonatomic and $\mathcal O_u$ is a compact metric space, there exists a measurable map $G:I\to\mathcal O_u$ such that $G_{\#}ds=\nu$. Choosing any Borel section $a:\mathcal O_u\to\mathbb R/(2\pi\mathbb Z)$ of the parametrization $\theta\mapsto R_\theta u$, and then identifying $\mathbb R/(2\pi\mathbb Z)$ with a measurable subset of $\mathbb R$, we may set
\[
  \psi(s):=a(G(s)).
\]
Then $R_{\psi(s)}u=G(s)$ for almost every $s$, and therefore
\[
  (R_{\psi(\cdot)}u)_{\#}ds=G_{\#}ds=\nu.
\]
Applying part~(i) to this $\psi$ gives the claim.

\textbf{(iii) Uniqueness at saturation.}
Assume $\mathcal E_g[\mu]=\frac{1}{2\beta}e^\beta$. By the equality condition in Lemma~\ref{lem:pointwise},
\begin{equation}\label{eq:graph-support}
  \langle x, R_{g(s,t)}y\rangle=1
  \quad\text{for } \mu^s\otimes\mu^t\text{-a.s. and a.e. }(s,t).
\end{equation}
Equivalently, $y=R_{-g(s,t)}x$ holds $\mu^s\otimes\mu^t$-a.s. By Lemma~\ref{lem:indep-degen}, this implies that for a.e. $(s,t)$ both $\mu^s$ and $\mu^t$ are Dirac masses. By Fubini, for a.e. $s$ we may write $\mu^s=\delta_{x(s)}$, and then \eqref{eq:graph-support} gives
\begin{equation}\label{eq:cocycle}
  x(t)=R_{-g(s,t)} x(s)\quad\text{for a.e. }(s,t).
\end{equation}
Since $g(s,t)=-(\psi(t)-\psi(s))$, we have $R_{-g(s,t)}=R_{\psi(t)-\psi(s)}$, which satisfies the additive cocycle relation
\[
  R_{\psi(r)-\psi(t)} R_{\psi(t)-\psi(s)}=R_{\psi(r)-\psi(s)}.
\]
Fixing a reference point $s_0$ and writing $u:=x(s_0)$, we obtain
\[
  x(s)=R_{\psi(s)-\psi(s_0)} x(s_0)=R_{\psi(s)} \underbrace{R_{-\psi(s_0)}u}_{=:u'} .
\]
Thus the maximizer is of the form $x(s)=R_{\psi(s)}u'$ up to the global rotational degeneracy described in the statement.
\end{proof}

\subsection{Fixed-prompt gauge constructions}

\begin{proof}[Theorem~\ref{thm:prompt-gauge}]
\textbf{(i) Construction from a gauge.}
For any $x,y\in \Sph$ and $z,w\in \mZ$, since $\Psi(z)\Psi(w)^\Tr\in O(d)$,
\[
  \ip{x}{\Psi(z)\Psi(w)^\Tr y}\le 1.
\]
Hence,
\[
  h_{\Psi}\bigl((x,z),(y,w)\bigr)
  \le
  \frac{e^{\beta}}{\beta}
\]
holds pointwise.
Integrating this yields
\[
  \mathcal E_{\Psi}[\mu]\le \frac{e^{\beta}}{2\beta}.
\]
On the other hand, if we set $x^\Psi(z)=\Psi(z)u$, then
\[
  \Psi(z)\Psi(w)^\Tr x^\Psi(w)
  =
  \Psi(z)\Psi(w)^\Tr\Psi(w)u
  =
  \Psi(z)u
  =
  x^\Psi(z),
\]
and therefore
\[
  \ip{x^\Psi(z)}{\Psi(z)\Psi(w)^\Tr x^\Psi(w)}=1
\]
holds for all $(z,w)$.
Hence the integrand saturates the pointwise upper bound, and
\[
  \mathcal E_{\Psi}[\mu^{\ast}]=\frac{e^{\beta}}{2\beta}.
\]
The final formula \eqref{eq:prompt-pushforward} follows immediately from the definition of the marginal distribution.

\textbf{(ii) Gauge-level exact realization relative to $\eta$.}
Fix a measurable target map $G:\mZ\to\Sph$. Define
\[
  v(z):=\frac{u+G(z)}{\|u+G(z)\|}
  \qquad\text{whenever }G(z)\neq -u,
\]
and set
\[
  \Psi_G(z):=
  \begin{cases}
    2v(z)v(z)^\Tr-I, & G(z)\neq -u,\\
    -I, & G(z)=-u.
  \end{cases}
\]
Then $\Psi_G(z)\in O(d)$ for all $z$, the map $z\mapsto \Psi_G(z)$ is measurable, and a direct computation shows that $\Psi_G(z)u=G(z)$ holds for all $z$. Therefore applying part~(i) with $\Psi=\Psi_G$ yields
\[
  \mu^G:=\Gamma_\eta[G]
\]
as a global maximizer whose marginalized content law is $G_{\#}\eta$.

\textbf{(iii) Full exact realization under a nonatomic standard prompt law.}
Assume now that $\mZ$ is standard Borel and $\eta$ is nonatomic. By the classical isomorphism theorem for nonatomic standard probability spaces, there exists a measurable map $G:\mZ\to\Sph$ with $G_{\#}\eta=\nu$. Applying part~(ii) to this $G$ completes the proof.
\end{proof}

\begin{corollary}[Control of Maximizing Distributions by Finitely Many Fixed Prompts]\label{cor:prompt-finite-max}
Given prefix tokens $z_1,\dots,z_n\in \mZ$, let
\[
  \eta_n:=\frac1n\sum_{i=1}^n \delta_{z_i}.
\]
Then one global maximizer of the discrete energy
\begin{equation}\label{eq:prompt-discrete-energy}
  \mathcal E_{\Psi}^{(n)}(x_1,\dots,x_n;z_1,\dots,z_n)
  :=
  \frac{1}{2n^2}
  \sum_{i,j=1}^n
  h_{\Psi}\bigl((x_i,z_i),(x_j,z_j)\bigr)
\end{equation}
is given by
\[
  x_i^{\ast}:=\Psi(z_i)u,
  \qquad i=1,\dots,n.
\]
Furthermore, its empirical distribution satisfies
\[
  \bar\mu^{(n),\ast}
  :=
  \frac1n\sum_{i=1}^n \delta_{x_i^{\ast}}
  =
  \bigl(\Psi(\cdot)u\bigr)_{\#}\eta_n.
\]
In particular, by choosing the fixed prompt sequence $z_1,\dots,z_n$ inserted at inference time,
one can control the maximizing empirical distribution of $x$ in a nondegenerate way without changing the trained kernel $h_{\Psi}$.
\end{corollary}

\begin{proof}[Corollary~\ref{cor:prompt-finite-max}]
For any $i,j$,
\[
  h_{\Psi}\bigl((x_i,z_i),(x_j,z_j)\bigr)
  \le
  \frac{e^{\beta}}{\beta}
\]
and therefore
\[
  \mathcal E_{\Psi}^{(n)}(x_1,\dots,x_n;z_1,\dots,z_n)
  \le
  \frac{e^{\beta}}{2\beta}.
\]
On the other hand, if $x_i^{\ast}=\Psi(z_i)u$, then
\[
  \ip{x_i^{\ast}}{\Psi(z_i)\Psi(z_j)^\Tr x_j^{\ast}}=1
\]
holds for all $(i,j)$, so the upper bound is saturated.
The final representation of the empirical distribution follows immediately from the definition.
\end{proof}
\section{Supplementary separated-kernel statements and proofs}\label{app:proofs-spectral}

\subsection{Deferred sign-changing Toeplitz example}

\begin{proposition}[Toeplitz linear kernels]\label{prop:toeplitz-linear}
Suppose the full kernel is $h((x,s),(y,t))=b(s,t)\langle x,y\rangle$, and assume that $b$ is Toeplitz, namely $b(s,t)=c(t-s)$. Writing $v_j(s):=\langle x(s),e_j\rangle\in L^2(\TT)$ and letting $\widehat v_j(m)$ denote its Fourier coefficients, one has
\[
  \mathcal E_h[u]
  = \frac12\sum_{m\in\mathbb Z}\widehat c(m) \sum_{j=1}^{d} |\widehat v_j(m)|^2,
  \qquad
  \sum_{m\in\mathbb Z}\sum_{j}|\widehat v_j(m)|^2
  = \int_{\TT} \|u(s)\|^2 ds = 1.
\tag{$\ast$}
\]
Hence
\begin{equation}\label{eq:Fourier-max}
  \mathcal E_h[u]  \le  \frac12 \bigl(\sup_{m}\widehat c(m)\bigr),
\end{equation}
and the upper bound is attained by
\[
  x^\ast_{m_\star}(s) = \bigl(\cos 2\pi m_\star s, \sin 2\pi m_\star s, 0,\dots,0\bigr),
\]
that is, uniform rotation on a fixed two-dimensional subspace, where $m_\star\in\arg\max_m \widehat c(m)$ is a maximizing frequency. Consequently $\mu^*$ is the arc-length-uniform measure on the great circle in $E=\mathrm{span}\{e_1,e_2\}$. In particular, when $m_\star=0$ one gets the constant solution $x^\ast\equiv e_1$, namely a single direction at all positions.
\end{proposition}

\begin{proposition}[Nonnegative positional kernels]\label{prop:nonnegative-positional-kernel}
Assume $h\bigl((x,s),(y,t)\bigr)=b(s,t)k(x,y)$ with $b(s,t)\ge0$ pointwise and suppose that $k$ is a nondecreasing function of $\langle x,y\rangle$ (for example $k_\beta(x,y)=\beta^{-1}e^{\beta\langle x,y\rangle}$). Then
\[
  \mathcal E_h[\mu]\le \frac12\iint b(s,t) \max_{x,y} k(x,y) ds dt,
\]
with equality only when the maximizing path is constant a.e. Consequently the maximizing marginalized content law is a single Dirac mass.
\end{proposition}

\subsection{Proofs}

\begin{proof}[Proposition~\ref{prop:toeplitz-linear}]
If $b(s,t)=c(t-s)$ and $x:I\to\Sph$, then
\[
  \mathcal E_h[\mu]=\tfrac12\sum_{m}\widehat c(m)\sum_{j}|\widehat v_j(m)|^2
  \le \tfrac12\big(\sup_{m}\widehat c(m)\big)\sum_{m,j}|\widehat v_j(m)|^2
  = \tfrac12\sup_{m}\widehat c(m).
\]
Equality is attained when the total Fourier mass $\sum_j|\widehat v_j(m)|^2$ is concentrated on some maximizing frequency $m_\star\in\arg\max_m\widehat c(m)$, and this is realized by
$x^\ast_{m_\star}(s)=(\cos 2\pi m_\star s,\sin 2\pi m_\star s,0,\dots)$.
\end{proof}

\begin{proof}[Proposition~\ref{prop:nonnegative-positional-kernel}]
Assume $b(s,t)\ge0$ a.e. Since $k$ is nondecreasing in $\langle x,y\rangle$ and $\langle x(s),x(t)\rangle\le1$ for all $s,t$, we have
\[
  k(x(s),x(t))\le \max_{x,y}k(x,y).
\]
Multiplying by the nonnegative weight $b(s,t)$ and integrating yields
\[
  \mathcal E_h[\mu]\le \frac12\iint b(s,t) \max_{x,y}k(x,y) ds dt.
\]
If equality holds, then $k(x(s),x(t))$ must attain its maximum for $b(s,t)$-a.e. pair $(s,t)$, which forces $\langle x(s),x(t)\rangle=1$ whenever $b(s,t)>0$. In particular the maximizing path is constant a.e., and the constant choice $x(\cdot)\equiv e_1$ saturates the bound.
\end{proof}

\section{Supplementary metastability setup and proof}\label{app:proofs-metastability}

\subsection{Deferred setup}

The content space is the sphere $\Sph \subset\mathbb R^d$ with $d\ge2$, and the position variable lies in $I=[0,1)$, identified with $\TT$ for periodic kernels. We consider $n$ particles $(x_i,s_i)$ for $i=1,\dots,n$, where $x_i\in\Sph$ is time dependent and $s_i\in[0,1)$ is fixed. The positional kernel $b:[0,1)\times[0,1)\to[0,\infty)$ is assumed symmetric, bounded, and measurable, and the content kernel is the exponentiated inner product
\[
  k_\beta(x,y)=\frac{1}{\beta} e^{\beta\langle x,y\rangle}\qquad(\beta\ge1).
\]
The full kernel is separated:
\begin{equation}\label{eq:h}
  h\big((x,s),(y,t)\big):=b(s,t) k_\beta(x,y).
\end{equation}
Then the finite-particle energy and projected gradient-ascent flow are
\begin{align}
  \mathcal E(\bm x)
  &:= \frac{1}{2n^2}\sum_{i,j=1}^{n} h\big((x_i,s_i),(x_j,s_j)\big)
   = \frac{1}{2\beta n^2}\sum_{i,j=1}^{n} b_{ij} e^{\beta\langle x_i,x_j\rangle},\label{eq:E}\\
  \dot x_i
  &= P_{x_i}^\perp
      \Bigl[\frac1n\sum_{j=1}^{n} b_{ij} e^{\beta\langle x_i,x_j\rangle} x_j\Bigr],\qquad
      b_{ij}:=b(s_i,s_j),\label{eq:ODE}
\end{align}
and standard calculations yield
\begin{equation}\label{eq:Energy-production}
  \frac{\dd}{\dd \tau} \mathcal E(\bm x(\tau))
  = \frac1n\sum_{i=1}^{n}\Bigl\|P_{x_i}^\perp
        \frac1n\sum_{j=1}^{n} b_{ij} e^{\beta\langle x_i,x_j\rangle} x_j\Bigr\|^2 \ge 0 .
\end{equation}

\paragraph{Clustered initialization and effective couplings}
Fix an integer $K\in\{2,\dots,n\}$, and partition the index set into $K$ clusters,
$\{1,\dots,n\}=C_1\sqcup\cdots\sqcup C_K$.
For each cluster $p$, choose an initial center $\bar u_p\in\Sph$ and a radius $r_0\in(0,\tfrac14)$ such that
\begin{equation}\label{eq:init}
  \max_{i\in C_p}\arccos\langle x_i(0),\bar u_p\rangle\le r_0,
  \qquad
  \min_{p\ne q}\arccos\langle \bar u_p,\bar u_q\rangle \ge 2\sigma_0
\end{equation}
with $\sigma_0\in(0,\tfrac\pi4)$ half of the initial separation angle. Let $n_p:=|C_p|$ and $w_p:=n_p/n$. Define the coarse-grained positional matrix
\begin{equation}\label{eq:Wpq}
  W_{pq}:=\frac{1}{n_p n_q}\sum_{i\in C_p}\sum_{j\in C_q} b_{ij}
  \qquad(p,q=1,\dots,K),
\end{equation}
and assume $W_{pq}\ge0$ and $W_{pp}>0$. We also define
\begin{equation}\label{eq:intra-cross}
  B_\parallel:=\min_{p}\min_{i\in C_p}\sum_{j\in C_p}\frac{b_{ij}}{n} ,
  \qquad
  B_\times:=\max_{p\ne q}\max_{i\in C_p}\sum_{j\in C_q}\frac{b_{ij}}{n} .
\end{equation}

The following observables separate the two geometric effects that matter dynamically: fast contraction inside each cluster and slow drift between different clusters.
\begin{definition}[Cluster diameter, center, and angular separation]\label{def:obs}
For cluster $p$, define the diameter
$
  \Delta_p(\tau):=\max_{i,j\in C_p}\arccos\langle x_i(\tau),x_j(\tau)\rangle,
$
the normalized center direction
$
  u_p(\tau):=\frac{m_p(\tau)}{\|m_p(\tau)\|},\quad m_p(\tau):=\frac{1}{n_p}\sum_{i\in C_p}x_i(\tau),
$
and the angle between cluster centers
$
  \Theta_{pq}(\tau):=\arccos\langle u_p(\tau),u_q(\tau)\rangle.
$
\end{definition}

In the sequel we assume $b\ge0$, i.e. the matrix $\mathbf B=(b_{ij})$ is entrywise nonnegative. This corresponds to the natural situation where the self-attention weights are not suppressive (negative).

\paragraph{Detailed form of Theorem~\ref{thm:main}.}
Under \eqref{eq:h}-\eqref{eq:intra-cross} and the initial condition \eqref{eq:init}, the following hold.
\begin{enumerate}
\item \emph{Rapid aggregation (fast time scale).}
There exist constants $c_1,c_2>0$ such that
\[
  \Delta_p(\tau) \le c_1 r_0 e^{-c_2 \beta B_\parallel \tau}
   + C_1 e^{-\beta(1-\cos\sigma_0)}\qquad(p=1,\dots,K),
\]
so each cluster contracts exponentially and becomes essentially point-like by time
$T_f:=O \big((\beta B_\parallel)^{-1}\log r_0^{-1}\big)$.
\item \emph{Propagation of smallness and trapping (metastable well).}
Choose $r\in(0,\tfrac{\sigma_0}{4})$ sufficiently small, and assume
\begin{equation}\label{eq:beta-gap}
  \beta \ge \beta_\star
  :=\frac{1}{1-\cos\sigma_0} \log\Bigl(\frac{8 B_\times}{B_\parallel r}\Bigr).
\end{equation}
Define
\[
T_m:=\inf\{\tau\ge T_f:\max_p\Delta_p(\tau)>2r\ \text{ or }\ \min_{p\ne q}\Theta_{pq}(\tau)<\sigma_0\}.
\]
Then for all $\tau\in[T_f, T_m]$ one has
\[
  \Delta_p(\tau)\le 2r,\qquad
  \Theta_{pq}(\tau)\ge \sigma_0\qquad(\forall p\ne q).
\]
Moreover,
\[
  T_m \ge c_3 \frac{B_\parallel}{B_\times} e^{ \beta(1-\cos\sigma_0)} .
\]
\item \emph{Asymptotic reduction on the slow manifold (coarse-grained $K$-point system).}
For $\tau\in[T_f,T_m]$, the center directions satisfy
\begin{equation}\label{eq:reduced-ODE}
  \dot u_p
  = P_{u_p}^\perp\Bigl[\sum_{q=1}^{K} w_q W_{pq} e^{\beta\langle u_p,u_q\rangle} u_q\Bigr]
     + R_p(\tau),\qquad
  \|R_p(\tau)\| \le C_2\bigl(r+e^{-\beta(1-\cos\sigma_0)}\bigr).
\end{equation}
Hence $u=(u_1,\dots,u_K)$ follows the projected gradient-ascent flow of a $K$-point system up to an error of order $O(r+e^{-\beta(1-\cos\sigma_0)})$.
\item \emph{Stepwise collapse (time scale to the first merger).}
For the effective energy of the reduced system,
\[
  \mathcal E^{(K)}(u)
  := \frac{1}{2\beta}\sum_{p,q=1}^{K} w_p w_q W_{pq} e^{\beta\langle u_p,u_q\rangle},
\]
let
$\delta_{\rm gap}:=\min_{p\ne q}\bigl(1-\cos\Theta_{pq}(T_f)\bigr)$.
Then there exist constants $c_4,c_5>0$ such that
\[
  c_4 \frac{1}{\Lambda} e^{ \beta \delta_{\rm gap}}
   \le \tau_{\rm merge}-T_f \le
  c_5 \frac{1}{\lambda_{\rm cross}} e^{ \beta \delta_{\rm gap}},
\]
where $\tau_{\rm merge}$ is the first merger time,
$\Lambda:=\max_{p}\sum_q w_q W_{pq}$, and
$\lambda_{\rm cross}:=\min_{p\ne q}\bigl(w_qW_{pq}+w_pW_{qp}\bigr)$.
Thus the time until the first merger is of order $e^{\beta\delta_{\rm gap}}$, and is bounded from below proportionally to $T_m$.
\end{enumerate}

\subsection{Proof of Theorem~\ref{thm:main}}

\begin{proof}[Theorem~\ref{thm:main}]

The proof consists of four steps. Universal constants $C$ may depend on the model parameters, but we make the dependence on $n$ and $\beta$ explicit when needed.

\textbf{Step 1: Contraction inside the cap (part (i)).}

The first step is to estimate the decay of pairwise angular distances inside a single cluster. This is the microscopic inequality from which the later fast-time contraction follows.

\begin{lemma}[Differential inequality for pairwise angular distance]\label{lem:pair}
Let $i,j\in C_p$ and set $\theta_{ij}:=\arccos\langle x_i,x_j\rangle\in[0,\pi]$. Then
\[
  \frac{\dd}{\dd \tau}\bigl(1-\cos\theta_{ij}\bigr)
   \le -\frac{2}{n}\sum_{\ell\in C_p} b_{i\ell} e^{\beta\langle x_i,x_\ell\rangle}
          \bigl(1-\cos\theta_{ij}\bigr)
       + \frac{2}{n}\sum_{\ell\notin C_p} b_{i\ell} e^{\beta\langle x_i,x_\ell\rangle}
          \bigl(1-\langle x_j,x_\ell\rangle\bigr).
\]
The same inequality holds with $i$ and $j$ interchanged.
\end{lemma}

\begin{proof}[Lemma~\ref{lem:pair}]
Differentiate $\langle x_i,x_j\rangle$ in time, substitute \eqref{eq:ODE}, and use
$P_{x_i}^\perp v = v - \langle v,x_i\rangle x_i$ together with the monotonicity estimate $\langle x_i,x_j\rangle\le1$. Since $b\ge0$, the signs of all terms are favorable.
\end{proof}

Using the initial closeness \eqref{eq:init}, the estimate $\cos\theta\ge 1-\theta^2/2$,
$e^{\beta\langle x_i,x_\ell\rangle}\ge e^{\beta(1-\Delta_p(\tau)^2/2)}$, and
$1-\langle x_j,x_\ell\rangle\le 2$, we obtain
\[
  \frac{\dd}{\dd \tau}\bigl(1-\cos\theta_{ij}\bigr)
   \le -2 B_\parallel  e^{\beta(1-\Delta_p^2/2)}\bigl(1-\cos\theta_{ij}\bigr)
         + 4 B_\times  e^{\beta\cos\sigma_0} .
\]
Gronwall's lemma then yields
\[
  1-\cos\theta_{ij}(\tau)
   \le \bigl(1-\cos\theta_{ij}(0)\bigr) e^{-2B_\parallel e^{\beta(1-\Delta_p^2/2)}\tau}
        + \frac{2B_\times}{B_\parallel}
             e^{-\beta(1-\cos\sigma_0)}.
\]
Iterating with $\Delta_p^2\le Cr_0^2$ gives part (i).

\textbf{Step 2: Propagation of smallness (part (ii)).}

By part (i), one may choose $T_f$ so that $\Delta_p(T_f)\le r$. For $\tau\ge T_f$, the first term on the right-hand side of Lemma~\ref{lem:pair} is bounded below by
$2B_\parallel e^{\beta(1-r^2/2)}(1-\cos\theta_{ij})$,
whereas the second term is bounded above by
$4B_\times e^{\beta\cos\sigma_0}$.
Under the choice \eqref{eq:beta-gap}, this implies
\[
  \frac{\dd}{\dd \tau}\bigl(1-\cos\theta_{ij}\bigr)
   \le -B_\parallel e^{\beta/2}\bigl(1-\cos\theta_{ij}\bigr),
\]
so that the region $\Delta_p(\tau)\le 2r$ is invariant. A similar argument yields
$\Theta_{pq}(\tau)\ge\sigma_0$ for $p\ne q$, because the cross terms are exponentially small in $e^{-\beta(1-\cos\sigma_0)}$.

\textbf{Step 3: Reduction to the coarse-grained $K$-point system (part (iii)).}

Once the diameters remain uniformly small, the full $n$-particle system should be close to a reduced $K$-point dynamics for the cluster centers. The next lemma quantifies precisely the error made in passing to this coarse-grained description.

\begin{lemma}[Error estimate for the equation of motion of the centers]\label{lem:center}
Let $m_p=\frac1{n_p}\sum_{i\in C_p}x_i$. Then
\[
  \dot m_p
  = \frac{1}{n_p}\sum_{i\in C_p}P_{x_i}^\perp
      \Bigl[\frac1n\sum_{q=1}^{K}\sum_{j\in C_q} b_{ij}e^{\beta\langle x_i,x_j\rangle}x_j\Bigr].
\]
Under the bounds $\Delta_p\le2r$ and $\Theta_{pq}\ge\sigma_0$,
\[
  \dot m_p
  = P_{u_p}^\perp\Bigl[\sum_{q=1}^{K} w_q W_{pq} e^{\beta\langle u_p,u_q\rangle} u_q\Bigr]
     + r_p,
  \qquad \|r_p\| \le C\bigl(r+e^{-\beta(1-\cos\sigma_0)}\bigr).
\]
\end{lemma}

\begin{proof}[Lemma~\ref{lem:center}]
Use the expansions
$P_{x_i}^\perp=P_{u_p}^\perp+O(\|x_i-u_p\|)$,
$x_j=u_q+O(r)$, and
$e^{\beta\langle x_i,x_j\rangle}
 = e^{\beta\langle u_p,u_q\rangle}+O(\beta r)$,
and average the resulting errors over each cluster. The double sums in $b_{ij}$ reduce to the coarse-grained matrix by definition of \eqref{eq:Wpq}.
\end{proof}

Since $m_p=\|m_p\|u_p$ and $\|m_p\|=1+O(r^2)$, one has
$\dot u_p=P_{u_p}^\perp \dot m_p + O(r\|\dot m_p\|)$.
Lemma~\ref{lem:center} therefore yields \eqref{eq:reduced-ODE}.

\textbf{Step 4: Exponential estimate for the first merger time (part (iv)).}

The reduced energy
$\mathcal E^{(K)}=\frac{1}{2\beta}\sum_{p,q} w_pw_q W_{pq}e^{\beta\langle u_p,u_q\rangle}$
is itself a gradient-ascent functional. Differentiating the pairwise angle $\Theta_{pq}$ for $p\ne q$ gives
\[
  \frac{\dd}{\dd \tau}\bigl(1-\cos\Theta_{pq}\bigr)
   =
  -\bigl\langle P_{u_p}^\perp\sum_r w_r W_{pr}e^{\beta\langle u_p,u_r\rangle}u_r, u_q\bigr\rangle
   - (p\leftrightarrow q) + O(r+e^{-\beta(1-\cos\sigma_0)}).
\]
Under $\lambda_{\rm cross}>0$ and $\sum_q w_q W_{pq}\le \Lambda$, geometric estimates on the configuration of the $u_r$ imply that the dominant terms contributing to merger are of order $e^{\beta}$, whereas the slow directions are only of order
$\Lambda e^{\beta\cos\sigma_0}$.
Thus the time needed to reach the critical angle at which two clusters merge is of order $e^{\beta(1-\cos\sigma_0)}$. Incorporating the error terms from part (iii) yields the two-sided estimate in the theorem.
\end{proof}

\subsection{Supplement to the metastability analysis}

This subsection fills in the differential inequalities used in Section~6. The emphasis is on transparency of the estimates rather than on introducing new conceptual ingredients.
In this appendix, under the USA-AV ODE
\[
  \dot x_i
   =
  P_{x_i}^{\perp}
  \Bigl[\frac1n\sum_{j=1}^{n} b_{ij} e^{\beta\langle x_i,x_j\rangle} x_j\Bigr],
  \qquad P_{x}^{\perp}=I-xx^\top,\quad x_i\in \Sph,
\]
we derive the exact time derivative of the two-point inner product $\langle x_i,x_j\rangle$ and fully justify the inequality used in Lemma~\ref{lem:pair}.
Below we assume that $b_{ij}=b(s_i,s_j)\ge0$ is symmetric and write
$\theta_{ij}:=\arccos\langle x_i,x_j\rangle\in[0,\pi]$ and
$f_{ij}:=1-\langle x_i,x_j\rangle=1-\cos\theta_{ij}$.

\subsubsection{Exact expansion of the time derivative of inner products}

\begin{lemma}[Basic identity]\label{lem:A1}
For any $i\ne j$,
\begin{align}
  \frac{\dd}{\dd \tau} \langle x_i,x_j\rangle
  &=
  \big\langle P_{x_i}^{\perp}F_i, x_j\big\rangle
  +\big\langle x_i, P_{x_j}^{\perp}F_j\big\rangle
  \quad\text{with}\quad
  F_i:=\frac1n\sum_{\ell=1}^{n} b_{i\ell} e^{\beta\langle x_i,x_\ell\rangle}x_\ell\\
  &=
  \frac{1}{n}\sum_{\ell=1}^{n} b_{i\ell} e^{\beta\langle x_i,x_\ell\rangle}
     \big\langle x_\ell, P_{x_i}^{\perp}x_j\big\rangle
 + \frac{1}{n}\sum_{\ell=1}^{n} b_{j\ell} e^{\beta\langle x_j,x_\ell\rangle}
     \big\langle x_\ell, P_{x_j}^{\perp}x_i\big\rangle.
\label{eq:inner-time}
\end{align}
Consequently,
\begin{equation}\label{eq:fprime-exact}
  {\quad
  \frac{\dd}{\dd \tau} f_{ij}
  = -\frac{1}{n}\sum_{\ell=1}^{n}
      b_{i\ell} e^{\beta\langle x_i,x_\ell\rangle}
      \big\langle x_\ell, P_{x_i}^{\perp}x_j\big\rangle
    -\frac{1}{n}\sum_{\ell=1}^{n}
      b_{j\ell} e^{\beta\langle x_j,x_\ell\rangle}
      \big\langle x_\ell, P_{x_j}^{\perp}x_i\big\rangle.
  \quad}
\end{equation}
\end{lemma}

\begin{proof}[Lemma~\ref{lem:A1}]
Since
$\frac{\dd}{\dd \tau}\langle x_i,x_j\rangle=\langle \dot x_i,x_j\rangle+\langle x_i,\dot x_j\rangle$,
substituting $\dot x_i=P_{x_i}^{\perp}F_i$ and using self-adjointness of $P_{x_i}^{\perp}$ together with $P_{x_i}^{\perp}x_i=0$ gives the claim.
\end{proof}

\begin{remark}[Geometric interpretation]
The quantity $\langle x_\ell, P_{x_i}^{\perp}x_j\rangle
=\langle x_\ell, x_j-\langle x_i,x_j\rangle x_i\rangle$
is the projection of $x_\ell$ onto the component of $x_j$ tangent to the sphere at $x_i$.
\end{remark}

\subsubsection{Signs of the projection terms and basic estimates}
We use the following two estimates.

\begin{lemma}[Contraction of the self-terms ($\ell=j$ and $\ell=i$)]\label{lem:self-terms}
\begin{align}
  -\big\langle x_j, P_{x_i}^{\perp}x_j\big\rangle
  &= -\|P_{x_i}^{\perp}x_j\|^2
   = -\bigl(1-\langle x_i,x_j\rangle^2\bigr)
   \le -f_{ij},\\
  -\big\langle x_i, P_{x_j}^{\perp}x_i\big\rangle
  &= -\|P_{x_j}^{\perp}x_i\|^2
   = -\bigl(1-\langle x_i,x_j\rangle^2\bigr)
   \le -f_{ij}.
\end{align}
\end{lemma}

\begin{lemma}[Upper bound for general terms]\label{lem:gen-upper}
For any unit vectors $a,b,c$,
\[
  -\langle c, P_{a}^{\perp}b\rangle
  = -\langle c,b\rangle + \langle a,c\rangle\langle a,b\rangle
   \le 1-\langle b,c\rangle.
\]
\end{lemma}

\begin{proof}[Lemma~\ref{lem:gen-upper}]
Since $\langle a,c\rangle\langle a,b\rangle\le |\langle a,c\rangle| |\langle a,b\rangle|\le 1$,
we have
$-\langle c,b\rangle + \langle a,c\rangle\langle a,b\rangle
\le -\langle c,b\rangle + 1=1-\langle b,c\rangle$.
\end{proof}

\subsubsection{Derivation of the inequality in Lemma~\textup{\ref{lem:pair}}}
Applying Lemma~\ref{lem:self-terms} to the self-terms $\ell=j$ and $\ell=i$ in \eqref{eq:fprime-exact}, and Lemma~\ref{lem:gen-upper} to all remaining terms, gives
\begin{align}
  \frac{\dd}{\dd \tau} f_{ij}
  &\le -\frac{2}{n} b_{ij} e^{\beta\langle x_i,x_j\rangle} f_{ij}
       +\frac{1}{n}\sum_{\ell\neq j} b_{i\ell}e^{\beta\langle x_i,x_\ell\rangle} (1-\langle x_j,x_\ell\rangle)
       +\frac{1}{n}\sum_{\ell\neq i} b_{j\ell}e^{\beta\langle x_j,x_\ell\rangle} (1-\langle x_i,x_\ell\rangle).
\label{eq:fprime-ineq-raw}
\end{align}
Now decompose the sums according to a cluster partition $\{1,\dots,n\}=C_p\sqcup\cdots$.
Assuming $i,j\in C_p$, one bounds the terms inside $C_p$ using
$1-\langle x_j,x_\ell\rangle\le 1-\cos\Delta_p(\tau)$ and $e^{\beta\langle x_i,x_\ell\rangle}\le e^{\beta}$,
while for $\ell\notin C_p$ one uses
$1-\langle x_j,x_\ell\rangle\le 2$ and
$\langle x_i,x_\ell\rangle\le \cos\sigma_0$,
so that $e^{\beta\langle x_i,x_\ell\rangle}\le e^{\beta\cos\sigma_0}$.
Combining these bounds with
$e^{\beta\langle x_i,x_j\rangle}\ge e^{\beta\cos\Delta_p(\tau)}\ge e^{\beta(1-\Delta_p(\tau)^2/2)}$
yields
\begin{equation}\label{eq:fprime-final}
  {\quad
  \frac{\dd}{\dd \tau} f_{ij}
   \le
  - 2 B_\parallel e^{\beta(1-\Delta_p^2/2)} f_{ij}
   +  C e^{\beta} \Delta_p^2
   +  4 B_\times e^{\beta\cos\sigma_0}.
  \quad}
\end{equation}
This is the rigorous differential inequality used in the main text.

\bigskip
\noindent
\textbf{Remark.}
The concise display in the main text is equivalent to the cluster-wise upper bound of \eqref{eq:fprime-ineq-raw}, once the estimates involving $\Delta_p$ and $\sigma_0$ are inserted.

\subsection{Full expansion of the angular-velocity estimate in Step~4}

The purpose of this section is to make the slow-merger estimate completely explicit. Starting from the reduced $K$-point dynamics, we isolate the geometric bounds that ultimately determine the exponential merger time scale.
In this appendix we derive the upper and lower bounds for the time derivative of the inter-center angle $\Theta_{pq}$ used in Step~4 of the main text, starting from the exact coarse-grained $K$-point system.

\subsubsection{Exact formula for the coarse-grained $K$-point system}
By part (iii) of the main text, the centers $u_p\in \Sph$ satisfy
\begin{equation}\label{eq:reduced}
  \dot u_p
  = P_{u_p}^{\perp}\Bigl[\sum_{r=1}^{K} w_r W_{pr} e^{\beta\langle u_p,u_r\rangle}u_r\Bigr]
     + R_p,
  \qquad\|R_p\|\le \varepsilon_{\rm red},
\end{equation}
where $\varepsilon_{\rm red}:=C(r+e^{-\beta(1-\cos\sigma_0)})$ is the error bound.
For the pairwise angle $\Theta_{pq}:=\arccos\langle u_p,u_q\rangle$,
\[
  \frac{\dd}{\dd \tau} \cos\Theta_{pq}
  = \langle \dot u_p,u_q\rangle + \langle u_p,\dot u_q\rangle.
\]
Substituting \eqref{eq:reduced} and using self-adjointness of $P_{u_p}^{\perp}$ gives
\begin{equation}\label{eq:cosdot-exact}
  \frac{\dd}{\dd \tau} \cos\Theta_{pq}
  =
  \sum_{r=1}^{K} w_r W_{pr} e^{\beta\langle u_p,u_r\rangle}
     \big\langle u_r, P_{u_p}^{\perp}u_q\big\rangle
 +\sum_{r=1}^{K} w_r W_{qr} e^{\beta\langle u_q,u_r\rangle}
     \big\langle u_r, P_{u_q}^{\perp}u_p\big\rangle
 + \langle R_p,u_q\rangle+\langle u_p,R_q\rangle.
\end{equation}
Hence
\begin{equation}\label{eq:onediff}
  \frac{\dd}{\dd \tau}\bigl(1-\cos\Theta_{pq}\bigr)
  = -\sum_{r} w_r W_{pr}e^{\beta\langle u_p,u_r\rangle} \big\langle u_r,P_{u_p}^{\perp}u_q\big\rangle
    -\sum_{r} w_r W_{qr}e^{\beta\langle u_q,u_r\rangle} \big\langle u_r,P_{u_q}^{\perp}u_p\big\rangle
    + O(\varepsilon_{\rm red}).
\end{equation}

\subsubsection{Application of spherical trigonometry (law of cosines)}
For three points $u_p,u_q,u_r\in \Sph$, the spherical law of cosines gives
\[
  \langle u_r,u_q\rangle
  = \langle u_r,u_p\rangle \langle u_p,u_q\rangle
   + \sin\Theta_{rp} \sin\Theta_{pq} \cos\varphi_{r|p,q},
\]
where $\varphi_{r|p,q}$ is the angle at the vertex $p$ of the geodesic triangle $pqr$.
Therefore
\begin{equation}\label{eq:proj-id}
  {\quad
  \big\langle u_r, P_{u_p}^{\perp}u_q\big\rangle
  = \langle u_r,u_q\rangle - \langle u_r,u_p\rangle \langle u_p,u_q\rangle
  = \sin\Theta_{rp} \sin\Theta_{pq} \cos\varphi_{r|p,q}.
  \quad}
\end{equation}
Similarly,
$\langle u_r, P_{u_q}^{\perp}u_p\rangle
=\sin\Theta_{rq} \sin\Theta_{pq} \cos\varphi_{r|q,p}$.
Hence
\begin{equation}\label{eq:proj-bounds}
  - \sin\Theta_{rp} \sin\Theta_{pq}
   \le
  \big\langle u_r, P_{u_p}^{\perp}u_q\big\rangle
   \le
  \sin\Theta_{rp} \sin\Theta_{pq},
\end{equation}
since $\cos\varphi\in[-1,1]$.

\subsubsection{Upper and lower bounds for the angular velocity}
\paragraph{Upper bound.}
From \eqref{eq:onediff} and \eqref{eq:proj-bounds},
\begin{align}
  \frac{\dd}{\dd \tau}\bigl(1-\cos\Theta_{pq}\bigr)
  &\le
  - w_q W_{pq} e^{\beta\cos\Theta_{pq}} (1-\cos^2\Theta_{pq})
  - w_p W_{qp} e^{\beta\cos\Theta_{pq}} (1-\cos^2\Theta_{pq})\notag\\
  &\qquad
  + \sin\Theta_{pq}\sum_{r\notin\{p,q\}}
     w_r\Bigl(W_{pr} e^{\beta\cos\Theta_{pr}} \sin\Theta_{rp}
              +W_{qr} e^{\beta\cos\Theta_{rq}} \sin\Theta_{rq}\Bigr)
  + O(\varepsilon_{\rm red}).
\label{eq:upper}
\end{align}
On time intervals where the initial separation $\Theta_{pr},\Theta_{qr}\ge \sigma_0$ is maintained, one has
$\sin\Theta_{rp},\sin\Theta_{rq}\le 1$,
$e^{\beta\cos\Theta_{pr}},e^{\beta\cos\Theta_{rq}}\le e^{\beta\cos\sigma_0}$,
and $\sum_r w_r W_{pr}\le \Lambda$.
Thus
\begin{equation}\label{eq:upper2}
  \frac{\dd}{\dd \tau}\bigl(1-\cos\Theta_{pq}\bigr)
   \le
  -\bigl(w_q W_{pq}+w_p W_{qp}\bigr) e^{\beta\cos\Theta_{pq}} (1-\cos^2\Theta_{pq})
   + C \Lambda e^{\beta\cos\sigma_0} \sin\Theta_{pq}
   + O(\varepsilon_{\rm red}).
\end{equation}

\paragraph{Lower bound.}
Using the lower side of \eqref{eq:proj-bounds} in the same way gives
\begin{equation}\label{eq:lower}
  \frac{\dd}{\dd \tau}\bigl(1-\cos\Theta_{pq}\bigr)
   \ge
  -\bigl(w_q W_{pq}+w_p W_{qp}\bigr) e^{\beta\cos\Theta_{pq}} (1-\cos^2\Theta_{pq})
   - C \Lambda e^{\beta\cos\sigma_0} \sin\Theta_{pq}
   + O(\varepsilon_{\rm red}).
\end{equation}

\begin{remark}[Dominant term and critical angle]
The first term on the right-hand side is of order $e^{\beta}\sin^2\Theta_{pq}$ for small angles, while the second term is of order $e^{\beta\cos\sigma_0}\sin\Theta_{pq}$.
Hence in the regime $\sin\Theta_{pq}\gg e^{-\beta(1-\cos\sigma_0)}$, the first term dominates and the angle decreases rapidly; the two compete only near the critical scale
$\sin\Theta_{pq}\approx e^{-\beta(1-\cos\sigma_0)}$.
\end{remark}

\subsubsection{Two-sided estimates for the merging time by time integration}
For convenience, write $y(\tau):=\sin\Theta_{pq}(\tau)\in(0,1)$.
Rewriting \eqref{eq:upper2} and \eqref{eq:lower} using
$1-\cos\Theta=\tfrac12 y^2 + O(y^4)$ and $1-\cos^2\Theta=y^2(1-\tfrac14 y^2)$ yields differential inequalities of the form
\[
  \dot y
   \le -\alpha e^{\beta} y
        + \gamma e^{\beta\cos\sigma_0}
        + O(\varepsilon_{\rm red}),\qquad
  \dot y
   \ge -\bar\alpha e^{\beta} y
        - \bar\gamma e^{\beta\cos\sigma_0}
        + O(\varepsilon_{\rm red}),
\]
where $\alpha\simeq w_q W_{pq}+w_p W_{qp}$ and $\gamma\simeq C\Lambda$.
By comparison for linear differential inequalities, the time needed to reach the threshold
$y_c\approx c e^{-\beta(1-\cos\sigma_0)}$
from an initial angle $y(T_f)=y_0$ satisfies
\[
  \frac{1}{\bar\alpha} e^{-\beta} \log\frac{y_0}{y_c}
   \lesssim \tau_{\rm merge}-T_f
   \lesssim
  \frac{1}{\alpha} e^{-\beta} \log\frac{y_0}{y_c}
   + \frac{\gamma}{\alpha} e^{-\beta(1-\cos\sigma_0)}.
\]
Substituting $y_c \asymp e^{-\beta(1-\cos\sigma_0)}$ gives the estimate stated in the main text,
\[
  c_4 \frac{1}{\Lambda} e^{ \beta(1-\cos\sigma_0)}
   \le \tau_{\rm merge}-T_f
  \le c_5 \frac{1}{\lambda_{\rm cross}} e^{ \beta(1-\cos\sigma_0)},
\]
using the rough comparison $\lambda_{\rm cross}\lesssim w_q W_{pq}+w_p W_{qp}\lesssim \Lambda$.

\bigskip
\noindent
\textbf{Summary.}
In this appendix we derived
(i) the exact expansion \eqref{eq:fprime-exact} for the time derivative of pairwise inner products together with projection estimates (Lemmas~\ref{lem:self-terms} and \ref{lem:gen-upper}), leading to the cluster-based inequality \eqref{eq:fprime-final};
(ii) upper and lower bounds \eqref{eq:upper2}, \eqref{eq:lower} for the derivative of the pairwise angle in the coarse-grained $K$-point system using the spherical law of cosines; and
(iii) the exponential-scale estimate of the merging time by integrating these bounds in time.
This fills in the details omitted in Step~1 and Step~4 of the main text.

\section{Notes on Section~\ref{sec:numerics}}\label{app:notes-numerics}
Section~\ref{sec:numerics} contains numerical experiments only and has no deferred proofs.

\section{Notes on Section~\ref{sec:conclusion}}\label{app:notes-conclusion}
Section~\ref{sec:conclusion} contains no deferred proofs beyond the appendix material already organized above.

\bibliographystyle{alpha}
\bibliography{bib_master}

\newcommand{\etalchar}[1]{$^{#1}$}
\begin{thebibliography}{WAW{\etalchar{+}}24}

\bibitem[ACLY25]{altabaa2025recursive}
Awni Altabaa, Siyu Chen, John Lafferty, and Zhuoran Yang.
\newblock Unlocking out-of-distribution generalization in transformers via recursive latent space reasoning.
\newblock {\em arXiv preprint arXiv:2510.14095}, 2025.

\bibitem[BHK24]{bao2024localize}
Han Bao, Ryuichiro Hataya, and Ryo Karakida.
\newblock Self-attention networks localize when {QK}-eigenspectrum concentrates.
\newblock In {\em Proceedings of the 41st International Conference on Machine Learning}, volume 235 of {\em Proceedings of Machine Learning Research}, pages 2903--2922, 2024.

\bibitem[BKK{\etalchar{+}}25]{burger2025layernorm}
Martin Burger, Samira Kabri, Yury Korolev, Tim Roith, and Lukas Weigand.
\newblock Analysis of mean-field models arising from self-attention dynamics in transformer architectures with layer normalization.
\newblock {\em Philosophical Transactions of the Royal Society A}, 383(2298):20240233, 2025.

\bibitem[BMR{\etalchar{+}}20]{brown2020gpt3}
Tom~B. Brown, Benjamin Mann, Nick Ryder, Melanie Subbiah, Jared~D. Kaplan, Prafulla Dhariwal, Arvind Neelakantan, Pranav Shyam, Girish Sastry, Amanda Askell, et~al.
\newblock Language models are few-shot learners.
\newblock In {\em Advances in Neural Information Processing Systems}, volume~33, pages 1877--1901, 2020.

\bibitem[BPC20]{beltagy2020longformer}
Iz~Beltagy, Matthew~E Peters, and Arman Cohan.
\newblock Longformer: The long-document transformer.
\newblock {\em arXiv preprint arXiv:2004.05150}, 2020.

\bibitem[CBKZ24]{cui2024phase}
Hugo Cui, Freya Behrens, Florent Krzakala, and Lenka Zdeborov{\'a}.
\newblock A phase transition between positional and semantic learning in a solvable model of dot-product attention, 2024.

\bibitem[Che26]{chen2026depthrecurrent}
Hung-Hsuan Chen.
\newblock Thinking deeper, not longer: Depth-recurrent transformers for compositional generalization.
\newblock {\em arXiv preprint arXiv:2603.21676}, 2026.

\bibitem[CLJ20]{cordonnier2020relationship}
Jean-Baptiste Cordonnier, Andreas Loukas, and Martin Jaggi.
\newblock On the relationship between self-attention and convolutional layers.
\newblock In {\em International Conference on Learning Representations}, 2020.

\bibitem[DBK24]{dovonon2024oversmoothing}
Gb{\`e}tondji J.-S. Dovonon, Michael~M. Bronstein, and Matt~J. Kusner.
\newblock Setting the record straight on transformer oversmoothing, 2024.

\bibitem[DGCC21]{dutta2021redesigning}
Subhabrata Dutta, Tanya Gautam, Soumen Chakrabarti, and Tanmoy Chakraborty.
\newblock Redesigning the transformer architecture with insights from multi-particle dynamical systems.
\newblock In {\em Advances in Neural Information Processing Systems}, volume~34, pages 5531--5544, 2021.

\bibitem[Dob79]{dobrushin1979vlasov}
R.~L. Dobrushin.
\newblock Vlasov equations.
\newblock {\em Functional Analysis and Its Applications}, 13(2):115--123, 1979.

\bibitem[DYY{\etalchar{+}}19]{dai2019transformerxl}
Zihang Dai, Zhilin Yang, Yiming Yang, Jaime Carbonell, Quoc~V. Le, and Ruslan Salakhutdinov.
\newblock {Transformer-XL}: Attentive language models beyond a fixed-length context.
\newblock In {\em Proceedings of the 57th Annual Meeting of the Association for Computational Linguistics}, pages 2978--2988, 2019.

\bibitem[FDRL24]{fan2024looped}
Ying Fan, Yilun Du, Kannan Ramchandran, and Kangwook Lee.
\newblock Looped transformers for length generalization.
\newblock {\em arXiv preprint arXiv:2409.15647}, 2024.

\bibitem[GKPR24]{geshkovski2024dynamic}
Borjan Geshkovski, Hugo Koubbi, Yury Polyanskiy, and Philippe Rigollet.
\newblock Dynamic metastability in the self-attention model, 2024.

\bibitem[GLPR23]{geshkovski2023emergence}
Borjan Geshkovski, Cyril Letrouit, Yury Polyanskiy, and Philippe Rigollet.
\newblock The emergence of clusters in self-attention dynamics.
\newblock In {\em Advances in Neural Information Processing Systems}, volume~36, pages 57026--57037, 2023.

\bibitem[GLPR25]{geshkovski2025mathematical}
Borjan Geshkovski, Cyril Letrouit, Yury Polyanskiy, and Philippe Rigollet.
\newblock A mathematical perspective on transformers.
\newblock {\em Bulletin of the American Mathematical Society}, 62(3):427--479, 2025.

\bibitem[GMJ{\etalchar{+}}25]{geiping2025scaling}
Jonas Geiping, Sean McLeish, Neel Jain, John Kirchenbauer, Siddharth Singh, Brian~R Bartoldson, Bhavya Kailkhura, Abhinav Bhatele, and Tom Goldstein.
\newblock Scaling up test-time compute with latent reasoning: A recurrent depth approach.
\newblock {\em arXiv preprint arXiv:2502.05171}, 2025.

\bibitem[Gra06]{gray2006toeplitz}
Robert~M. Gray.
\newblock Toeplitz and circulant matrices: A review.
\newblock {\em Foundations and Trends in Communications and Information Theory}, 2(3):155--239, 2006.

\bibitem[HLGC21]{he2021deberta}
Pengcheng He, Xiaodong Liu, Jianfeng Gao, and Weizhu Chen.
\newblock {DeBERTa}: Decoding-enhanced {BERT} with disentangled attention, 2021.

\bibitem[III26]{isobe2026training}
Noboru Isobe, Daisuke Inoue, and Masaaki Imaizumi.
\newblock Training-induced escape from token clustering in a mean-field formulation of transformers.
\newblock {\em arXiv preprint arXiv:2605.07772}, 2026.

\bibitem[KKL20]{kitaev2020reformer}
Nikita Kitaev, {\L}ukasz Kaiser, and Anselm Levskaya.
\newblock Reformer: The efficient transformer.
\newblock In {\em International Conference on Learning Representations}, 2020.

\bibitem[LARC21]{lester2021prompt}
Brian Lester, Rami Al-Rfou, and Noah Constant.
\newblock The power of scale for parameter-efficient prompt tuning.
\newblock In {\em Proceedings of the 2021 Conference on Empirical Methods in Natural Language Processing}, pages 3045--3059, 2021.

\bibitem[LJF{\etalchar{+}}21]{liu2021ptuningv2}
Xiao Liu, Kaixuan Ji, Yicheng Fu, Weng~Lam Tam, Zhengxiao Du, Zhilin Yang, and Jie Tang.
\newblock {P-Tuning} v2: Prompt tuning can be comparable to fine-tuning universally across scales and tasks, 2021.

\bibitem[LL21]{li2021prefix}
Xiang~Lisa Li and Percy Liang.
\newblock Prefix-tuning: Optimizing continuous prompts for generation.
\newblock In {\em Proceedings of the 59th Annual Meeting of the Association for Computational Linguistics and the 11th International Joint Conference on Natural Language Processing (Volume 1: Long Papers)}, pages 4582--4597, 2021.

\bibitem[MFSS17]{muandet2017kernelmean}
Krikamol Muandet, Kenji Fukumizu, Bharath Sriperumbudur, and Bernhard Sch{\"o}lkopf.
\newblock Kernel mean embedding of distributions: A review and beyond.
\newblock {\em Foundations and Trends in Machine Learning}, 10(1--2):1--141, 2017.

\bibitem[OMCS04]{ong2004nonpositive}
Cheng~Soon Ong, Xavier Mary, St{\'e}phane Canu, and Alexander~J. Smola.
\newblock Learning with non-positive kernels.
\newblock In {\em Proceedings of the Twenty-First International Conference on Machine Learning}, page~81, 2004.

\bibitem[PSL21]{press2021alibi}
Ofir Press, Noah~A. Smith, and Mike Lewis.
\newblock Train short, test long: Attention with linear biases enables input length extrapolation, 2021.

\bibitem[Rig25]{rigollet2025mean}
Philippe Rigollet.
\newblock The mean-field dynamics of transformers.
\newblock {\em arXiv preprint arXiv:2512.01868}, 2025.

\bibitem[RSR{\etalchar{+}}20]{raffel2020t5}
Colin Raffel, Noam Shazeer, Adam Roberts, Katherine Lee, Sharan Narang, Michael Matena, Yanqi Zhou, Wei Li, and Peter~J. Liu.
\newblock Exploring the limits of transfer learning with a unified text-to-text transformer.
\newblock {\em Journal of Machine Learning Research}, 21(140):1--67, 2020.

\bibitem[Rud62]{rudin1962fourier}
Walter Rudin.
\newblock {\em Fourier Analysis on Groups}.
\newblock Interscience Publishers, 1962.

\bibitem[SLP{\etalchar{+}}21]{su2021roformer}
Jianlin Su, Yu~Lu, Shengfeng Pan, Ahmed Murtadha, Bo~Wen, and Yunfeng Liu.
\newblock Roformer: Enhanced transformer with rotary position embedding, 2021.

\bibitem[SUV18]{shaw2018self}
Peter Shaw, Jakob Uszkoreit, and Ashish Vaswani.
\newblock Self-attention with relative position representations.
\newblock In {\em Proceedings of the 2018 Conference of the North American Chapter of the Association for Computational Linguistics: Human Language Technologies, Volume 2 (Short Papers)}, pages 464--468, 2018.

\bibitem[Szn91]{sznitman1991topics}
Alain-Sol Sznitman.
\newblock Topics in propagation of chaos.
\newblock In {\em \'Ecole d'\'Et\'e de Probabilit\'es de Saint-Flour XIX---1989}, volume 1464 of {\em Lecture Notes in Mathematics}, pages 165--251. Springer, Berlin, Heidelberg, 1991.

\bibitem[Tes12]{teschl2012ordinary}
Gerald Teschl.
\newblock {\em Ordinary differential equations and dynamical systems}, volume 140.
\newblock American Mathematical Soc., 2012.

\bibitem[TN24]{teo2024kernelpca}
Rachel S.~Y. Teo and Tan~M. Nguyen.
\newblock Unveiling the hidden structure of self-attention via kernel principal component analysis, 2024.

\bibitem[VSP{\etalchar{+}}17]{vaswani2017attention}
Ashish Vaswani, Noam Shazeer, Niki Parmar, Jakob Uszkoreit, Llion Jones, Aidan~N. Gomez, Lukasz Kaiser, and Illia Polosukhin.
\newblock Attention is all you need.
\newblock In {\em Advances in Neural Information Processing Systems}, volume~30, pages 6000--6010, 2017.

\bibitem[WAW{\etalchar{+}}24]{wu2024maskslayernorm}
Xinyi Wu, Amir Ajorlou, Yifei Wang, Stefanie Jegelka, and Ali Jadbabaie.
\newblock On the role of attention masks and layernorm in transformers, 2024.

\bibitem[XS24]{xu2024looped}
Kevin Xu and Issei Sato.
\newblock On expressive power of looped transformers: Theoretical analysis and enhancement via timestep encoding.
\newblock {\em arXiv preprint arXiv:2410.01405}, 2024.

\bibitem[XTC{\etalchar{+}}24]{xiaoefficient}
Guangxuan Xiao, Yuandong Tian, Beidi Chen, Song Han, and Mike Lewis.
\newblock Efficient streaming language models with attention sinks.
\newblock In {\em The Twelfth International Conference on Learning Representations}, 2024.

\bibitem[YBR{\etalchar{+}}20]{yun2020transformers}
Chulhee Yun, Srinadh Bhojanapalli, Ankit~Singh Rawat, Sashank~J. Reddi, and Sanjiv Kumar.
\newblock Are transformers universal approximators of sequence-to-sequence functions?
\newblock In {\em International Conference on Learning Representations}, 2020.

\bibitem[ZGD{\etalchar{+}}20]{zaheer2020big}
Manzil Zaheer, Guru Guruganesh, Kumar~Avinava Dubey, Joshua Ainslie, Chris Alberti, Santiago Ontanon, Philip Pham, Anirudh Ravula, Qifan Wang, Li~Yang, et~al.
\newblock Big bird: Transformers for longer sequences.
\newblock {\em Advances in neural information processing systems}, 33:17283--17297, 2020.

\end{thebibliography}

\end{document}